\documentclass[review,10pt]{elsarticle}
\usepackage{natbib}
\usepackage{lineno,hyperref}
\usepackage{pifont}

\usepackage{graphicx}
\usepackage{float}
\usepackage{subfig}
\modulolinenumbers[5]

\journal{Journal of \LaTeX\ Templates}

\makeatletter
\def\ps@headings{%
	\def\@oddhead{\mbox{}\scriptsize\rightmark \hfil \thepage}%
	\def\@evenhead{\scriptsize\thepage \hfil \leftmark\mbox{}}%
	\def\@oddfoot{}%
	\def\@evenfoot{}}
\makeatother 
\usepackage{epsfig}
\usepackage{url}
\newtheorem{definition}{Definition}
\newtheorem{proposition}{Proposition}

\newtheorem{proof}{Proof}

\usepackage{algorithmicx}
\usepackage{algpseudocode}
\usepackage{algorithm}
\usepackage{color}
\usepackage{amsmath}
\usepackage{amssymb}
\usepackage{multirow}
\usepackage{colortbl}
\usepackage{xcolor}

\newcommand{\warn}[1]{}

\begin{document}

\title{\textit{FACM}: Intermediate Layer Still Retain Effective Features against Adversarial Examples}

\author[mymainaddress]{Xiangyuan Yang}\ead{ouyang\_xy@stu.xjtu.edu.cn}

\author[mymainaddress]{Jie Lin\corref{mycorrespondingauthor}}
\cortext[mycorrespondingauthor]{Corresponding author}\ead{jielin@mail.xjtu.edu.cn}

\author[mysecondaryaddress]{Hanlin Zhang}\ead{ hanlin@qdu.edu.cn}
\author[mymainaddress]{Xinyu Yang}\ead{yxyphd@mail.xjtu.edu.cn}

\author[mymainaddress]{Peng Zhao}\ead{p.zhao@mail.xjtu.edu.cn}

\address[mymainaddress]{School of Computer Science and Technology, Xi'an Jiaotong University, Xi'an, China}
\address[mysecondaryaddress]{Qingdao University, Qingdao, China}

\begin{abstract}
In strong adversarial attacks against deep neural networks (DNN), the generated adversarial example will mislead the DNN-implemented classifier by destroying the output features of the last layer. To enhance the robustness of the classifier, in our paper, a \textbf{F}eature \textbf{A}nalysis and \textbf{C}onditional \textbf{M}atching prediction distribution (FACM) model is proposed to utilize the features of intermediate layers to correct the classification. Specifically, we first prove that the intermediate layers of the classifier can still retain effective features for the original category, which is defined as the correction property in our paper. According to this, we propose the FACM model consisting of \textbf{F}eature \textbf{A}nalysis (FA) correction module, \textbf{C}onditional \textbf{M}atching \textbf{P}rediction \textbf{D}istribution (CMPD) correction module and decision module. The FA correction module is the fully connected layers constructed with the output of the intermediate layers as the input to correct the classification of the classifier. The CMPD correction module is a conditional auto-encoder, which can not only use the output of intermediate layers as the condition to accelerate convergence but also mitigate the negative effect of adversarial example training with the Kullback-Leibler loss to match prediction distribution. Through the empirically verified diversity property, the correction modules can be implemented synergistically to reduce the adversarial subspace. Hence, the decision module is proposed to integrate the correction modules to enhance the DNN classifier's robustness. Specially, our model can be achieved by fine-tuning and can be combined with other model-specific defenses. Finally, the extended experiments demonstrate the effectiveness of our FACM model against adversarial attacks, especially optimization-based white-box attacks, and query-based black-box attacks, in comparison with existing models and methods.

\end{abstract}


\begin{keyword}
Adversarial samples\sep Correction model\sep Correction property\sep Diversity property\sep Deep neural network
\end{keyword}


\maketitle

\section{Introduction}
\label{sec:intro}
In Deep neural networks (DNN), adversarial examples that are perturbed imperceptibly from original samples may lead to the misclassification of DNN-implemented classifiers, which may result in serious threats to the DNN applications, such as self-driving cars, finance, and medical diagnosis, etc. Hence, improving the robustness of the DNN-implemented classifier is necessary to eliminate such serious threats.

At present, considerable efforts had been developed to improve the robustness of the DNN-implemented classifier against adversarial examples, which can be categorized as adversarial training (AT)~\cite{PGD,ATHE,Fast-AT,FAT,ATTA,YOPO,FRL,TRADES}, randomization~\cite{SAP,Mitigating-adversarial-effects,RS,SLQ,Adversarial-noise-layer,Random-self-ensemble,Randomized-diversification,SORS} and input purification~\cite{PixelDefend,Defense-GAN}. To be specific, in adversarial training~\cite{PGD}, the clean examples and their corresponding adversarial examples are used synergistically to train the DNN to enhance the robustness. Several methods have been further developed to accelerate adversarial training~\cite{Fast-AT,YOPO} and mitigate low efficiency~\cite{ATHE,FAT,ATTA}, low generalization~\cite{TRADES} and unfairness~\cite{FRL}. In randomization, the random operations, e.g., random imaging resizing and padding~\cite{RS} and randomized diversification~\cite{Randomized-diversification}, etc., are used to disturb the process of the adversarial example generation. In input purification, the extra generative model, e.g., PixelCNN~\cite{PixelCNN} and GAN~\cite{GAN}, is used to purify the input of the classifier, thereby enhancing the robustness. However, these methods still achieve low robustness to adversarial examples perturbed by less disturbance in optimization-based white-box attacks or generated without the architecture and parameters of the victim model in query-based black-box attacks. Hence, this call for a method that can improve the robustness of the DNN-implemented classifier against the adversarial examples generated by both optimization-based white-box attacks and query-based black-box attacks.


To address this issue, this paper propose a novel Feature Analysis and Conditional Matching prediction distribution (FACM) model to improve the robustness of the DNN-implemented classifier against both optimization-based white-box attacks and query-based black-box attacks. First, we find that the intermediate layers of the classifier may still retain effective features for the original category, which is defined as \textit{Correction Property} in our model. Then, based on the \textit{Correction Property}, the Feature Analysis (FA) correction module, the Conditional Matching Prediction Distribution (CMPD) correction module, and the decision module are proposed to compose the FACM model. The FA correction module is built by the retrained features in the intermediate layers to correct the classification of the perturbed input of the classifier. The CMPD correction module is a conditional auto-encoder, which can achieve fast convergence by adding the features of intermediate layers as the condition, thereby mitigating the negative effect of adversarial examples effectively. In addition, we empirically find that the diversity exists in the correction modules and can be enhanced as the attack strength increases, which is defined as the \textit{Diversity Property} in our model. According to the \textit{Diversity Property}, the decision module is proposed to integrate all correction modules in our FACM model to further enhance the robustness of the DNN-implemented classifier against both optimization-based white-box attacks and query-based black-box attacks by reducing the adversarial subspace.


Our main contributions are summarized as:
\begin{itemize}
    \item The Feature Analysis and Conditional Matching prediction distribution (FACM) model is proposed to improve the robustness of the DNN-implemented classifier against adversarial examples, especially generated by optimization-based white-box attacks and query-based black-box attacks. In addition, the FACM model can also be embedded into the classifiers with advanced defense methods.
   
    \item The \textit{Correction Property} that is the output of the intermediate layers in the classifier can still retain effective features for the original category against adversarial examples is proposed and analyzed. Based on the \textit{Correction Property}, the FA and CMPD correction modules are proposed to be collaboratively implemented in our FACM model to enhance the robustness of the DNN-implemented classifier.
    
    \item The \textit{Diversity Property} that is the diversity is existed in the multiple FA and CMPD correction modules against adversarial examples and can be enhanced as the attack strength increases are empirically proved in our FACM model. Based on the \textit{Diversity Property}, the decision module is proposed to integrate the correction modules to further enhance the robustness of the DNN-implemented classifier.
    
     \item The Extensive experiments demonstrate that our FACM model can significantly improve the robustness of the classifier against adversarial examples, especially generated by optimization-based white-box attacks and query-based black-box attacks. The results also show that our FACM model can be compatible with existing advanced defense methods.
\end{itemize}

\begin{table}
	\centering \caption{Notation}\label{Notation}
	\begin{tabular}{lp{0.7\textwidth}} \hline
		$x:$& The original clean example.\\
		$y:$& The corresponding ground truth label of $x$.\\
		$\mathcal{D}:$& The distribution of the training dataset.\\
		$f:$& The DNN-implemented classifier.\\ 
		$\mathcal{H}:$& The hypothesis space of the classifier $f$.\\
		$L(\cdot, \cdot):$&  The loss function (e.g., the cross-entropy loss).\\
		$\hat{x}=x+\delta:$& The perturbed image with added perturbation $\delta$.\\
		$\varPhi(\cdot),\varPsi(\cdot):$& The encoder and decoder of autoencoder.\\
		$\theta^{\varPhi},\theta^{\varPsi}:$& The parameters of $\varPhi(\cdot)$ and $\varPsi(\cdot)$.\\
		$l_i(\cdot):$& The output of the $i^{th}$ intermediate layer of the classifier $f$.\\
		$f_i(\cdot):$& The $i^{th}$ auxiliary classifier with $l_i(x)$ as the input.\\
		$\varphi_i(\cdot):$& The $i^{th}$ FA correction module.\\
		$n-1:$& The number of intermediate layers in the classifier $f$.\\
		$g_i(\cdot):$& The $i^{th}$ conditional autoencoder.\\
		$f\circ g_i(\cdot):$& The $i^{th}$ CMPD correction module.\\
		$h(\cdot):$& The decision module.\\
		$\tau:$& The number of selected correction modules to predict the input.\\
		$\mathbb{C}:$& The correction module set, including the classifier $f$.\\
		$c(\cdot):$& A correction module in $\mathbb{C}$.\\
		$F(x):$& The output of the FACM model to the input $x$.\\
		$S_i:$& The classification sequence of the top $i$ auxiliary classifiers.\\
		$\mathcal{V}:$& The input space of the classifier $f$.\\
		$\mathcal{V}_{S_i}:$& The subspace with the classification sequence $S_i$ in the space $\mathcal{V}$.\\
		$\xi_i(\cdot):$& The outputs concatenation of the top $i$ auxiliary classifiers.\\
		\hline
	\end{tabular}
\end{table}

The rest of the paper is organized as follows: several advanced defense methods are briefly introduced in Section~\ref{sec:preliminaries}. Our proposed FACM model is described in Section~\ref{sec:proposed_approach}. The experiments are conducted in Section~\ref{sec:experiments}. We review the related works in Section~\ref{sec:related_work} and conclude this paper in Section~\ref{sec:conclusion}.

\section{Preliminaries}
\label{sec:preliminaries}
In this section, the standard adversarial training, TRADES, and matching prediction distribution, are mentioned briefly. All notations are defined in Table~\ref{Notation}.

\textbf{Standard Adversarial Training}: In adversarial training, the adversarial attacks is used to generate adversarial examples to achieve data augmentation, and the DNN model is trained by both the orignal examples and adversarial examples to improve the robustness. For instance, the PGD method~\cite{PGD} firstly formulate the adversarial training as a min-max optimization problem, which can be represented as%
\begin{align}\label{equation:adversarial_training}
\underset{f\in \mathcal{H}}{\min}\mathbb{E} _{\left( x,y \right) \sim \mathcal{D}}\left[ \underset{\delta \in \mathcal{B}}{\max}L\left( f\left( x+\delta \right) ,y \right) \right]
\end{align}
where $\mathcal{H}$ is the hypothesis space of the classifier $f$, $\mathcal{D}$ is the distribution of the training dataset, $L$ is a loss function, and $\mathcal{B}$ is the available perturbation space that is usually represented as a $l_\infty$ norm ball around $x$. The basic ideal of adversarial training is to find a perturbed image $x+\delta$ based on a given original image $x$ to maximize the loss with respect to correct classification. 

\textbf{TRADES}: To trade off natural and robust errors, TRASES method~\cite{TRADES} is proposed to train the DNN model by both natural and adversarial examples and change the min-max formulation as:
\begin{align}
\label{equation:trades}
\underset{f\in \mathcal{H}}{\min}\mathbb{E} _{\left( x,y \right) \sim \mathcal{D}}\left[ \begin{array}{c}
L\left( f\left( x \right) ,y \right)\\
+\beta \cdot \underset{\delta \in \mathcal{B}}{\max}L_{KL}\left( f\left( x+\delta \right) ,f\left( x \right) \right)\\
\end{array} \right]
\end{align}
where $L_{KL}$ is the Kullback-Leibler loss, $\beta$ is a regularization parameter that controls the trade-off between standard accuracy and robustness. As shown in Eq.~\ref{equation:trades}, when $\beta$ increases, standard accuracy will decrease while robustness will increase, and vice versa.

\textbf{Matching Prediction Distribution}: To correct the output of the DNN against adversarial examples, MPD~ \cite{MPD} is proposed to conduct a novel adversarial correction model, which involves an autoencoder trained by a custom loss function that is generated by the Kullback-Leibler divergence between the classifier predictions on the original and reconstructed instances:
\begin{align}
\underset{\theta ^{\varPhi},\theta ^{\varPsi}}{\min}\mathbb{E} _{\left( x,y \right) \sim \mathcal{D}}\left[ L_{KL}\left( f\left( x \right) ,f\left( \varPsi \left( \varPhi \left( x \right) \right) \right) \right) \right]
\end{align}
where $\varPhi(\cdot)$ and $\varPsi(\cdot)$ are the encoder and decoder of autoencoder, respectively, $\theta ^{\varPhi}$ and $\theta ^{\varPsi}$ are the parameters of $\varPhi(\cdot)$ and $\varPsi(\cdot)$ that need to be optimized. This method is unsupervised, easy to train and does not require any knowledge about the underlying attack.

\begin{figure*}[ht]
\centering
\centerline{\includegraphics[width=\textwidth]{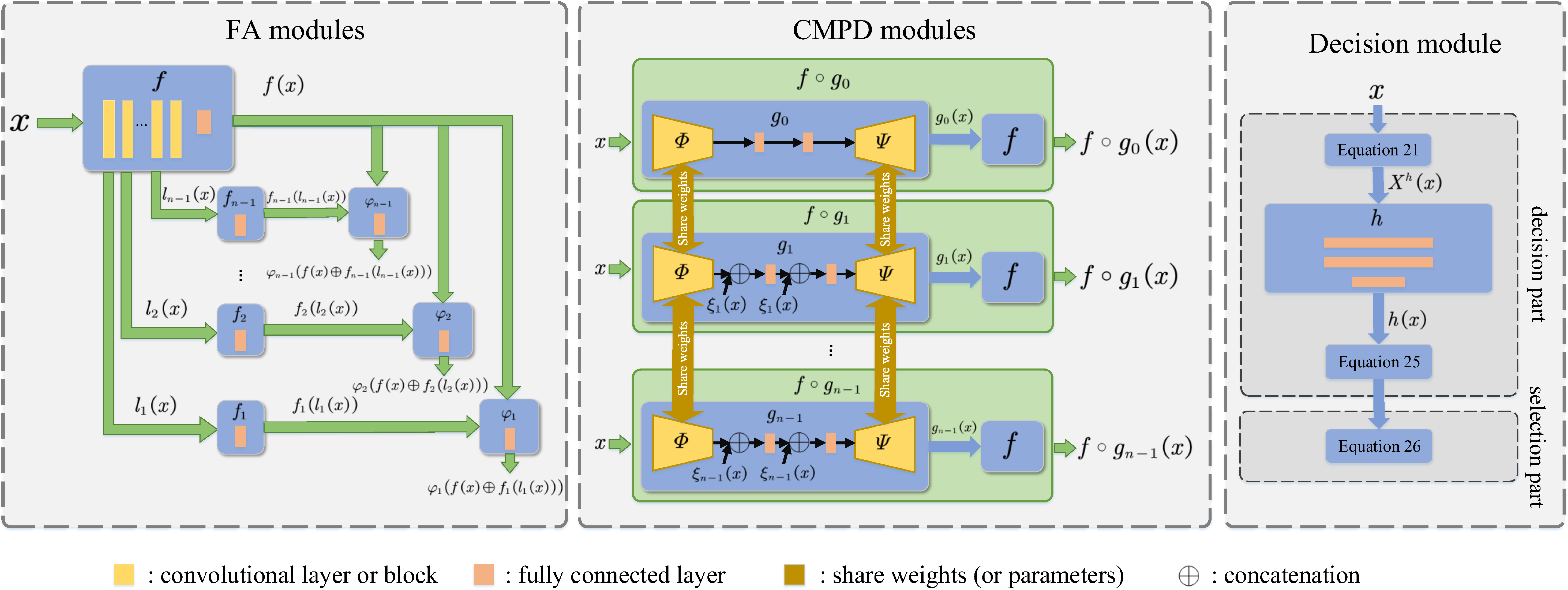}}
\caption{A framework of the FACM model.}
\label{framework_of_FACM}
\end{figure*}

\section{Our Approach}
\label{sec:proposed_approach}
In this section, our \textbf{F}eature \textbf{A}nalysis and \textbf{C}onditional \textbf{M}atching prediction distribution (FACM) model is presented, which consists of three modules: the feature analysis (FA) correction module, the conditional matching prediction distribution (CMPD) correction module, and the decision module, as shown in Fig.~\ref{framework_of_FACM}. Firstly, the construction of our FACM model is presented. Then, each module is described in detail. Finally, we describe the operation process of the FACM model.

\subsection{The Construction of the FACM Model}
\label{sec:FACM}

Firstly, the FA correction module denoted as $\varphi_i(\cdot)$ where $1\leq i\leq n-1$, is constructed. In the FA correction module, an auxiliary classifier $f_i$ is embedded into the DNN-implemented classifier to correct the classification by utilizing the features extracted by the intermediate layer of the DNN classifier.

Secondly, the CMPD correction module denoted as $f\circ g_i(\cdot)$ where $0\leq i\leq n-1$, is constructed. To mitigate the negative effect of the perturbation, a conditional autoencoder, which is conditioned with the outputs of the top $i$ auxiliary classifiers, is built to reconstruct the perturbed input. The loss function of the conditional autoencoder is the Kullback-Leibler divergence between the predictions of the DNN classifier on the original and reconstructed inputs. Finally, the conditional autoencoder is embedded into the DNN classifier to construct the CMPD correction module.

Lastly, the decision module includes the decision component and the selection component. The decision component, denoted as $h(\cdot)$, is constructed with the concatenation of outputs of all auxiliary classifiers as the input and determines the weights of FA and CMPD correction modules in the FACM model. According to the determined weights, the selection component randomly chooses $\tau$ number of correction modules to correct the classification of the DNN classifier by averaging the output of these selected FA and CMPD correction modules.

\subsection{The Feature Analysis (FA) correction module}
\label{sec:FA}
In our model, the feature analysis (FA) correction model uses the features retained by the intermediate layer to correct the outputs of the DNN classifier, thereby improving the robustness. Specifically, the auxiliary classifier is embedded into the DNN classifier to constitute the FA correction module. 

The input of the auxiliary classifier is the output of the intermediate layer of the DNN. Different intermediate layers will build different auxiliary classifiers. Hence, the $i^{th}$ auxiliary classifier is defined in Definition~\ref{AC}.

\begin{definition}[Auxiliary Classifier]\label{AC}
For an instance $(x,y)\sim \mathcal{D}$, the $i^{th}$ auxiliary classifier, denoted as $f_i(\cdot)$, is built by taking the output of the $i^{th}$ intermediate layer of the DNN as the input, which can be fine-tuned by:
\begin{align}
\label{equation:objective_function_of_auxiliary_classifier}
\underset{\theta ^{f_i}}{\min}\mathbb{E} _{\left( x,y \right) \sim \mathcal{D}}\left[ L_{CE}\left( f_i\left( l_i\left( x \right) \right) ,y \right) \right]
\end{align}
where $L_{CE}(\cdot, \cdot)$ is the cross-entropy loss, $\theta^{f_i}$ represents the parameters of the $i^{th}$ auxiliary classifier $f_i$, and $l_i\left( x \right)$ is the $i^{th}$ intermediate layer of the DNN.
\end{definition}

Proposition~\ref{prop:effectiveness_of_middle_layer} demonstrates that the proposed auxiliary classifier can correctly predict adversarial examples that mislead the DNN-implemented classifier. To prove the Proposition~\ref{prop:effectiveness_of_middle_layer}, the classification sequence and corresponding space are defined in Definitions~\ref{def:classification_sequence} and \ref{def:classification_sequence_subspace} respectively and can be used to  analyze the effectiveness of intermediate layer's features against adversarial examples.

\begin{definition}[Classification sequence]
	\label{def:classification_sequence}
	For an instance $(x,y)\sim \mathcal{D}$, the classification sequence $S_i$ of the top $i$ auxiliary classifiers can be denoted as:
	\begin{gather}
	S_i = (argmax(f_1(l_1(x))),\nonumber\\
	argmax(f_2(l_2(x))), \cdots, argmax(f_i(l_i(x)))).
	\end{gather}
\end{definition}

\begin{definition}[Classification sequence space]
	\label{def:classification_sequence_subspace}
	In the input space $\mathcal{V}$ of the classifier, the classification sequence space $\mathcal{V}_{S_i}$ indicates
	\begin{gather}
	(argmax(f_1(l_1(x))), argmax(f_2(l_2(x))), \nonumber\\
	\cdots, argmax(f_i(l_i(x))))=S_i,\quad if \quad x\in \mathcal{V}_{S_i}.\label{equation:classification_sequence_space}\\
	\mathcal{V}_{S_i}\subset \mathcal{V}. \label{equation:classification_sequence_subspace}
	\end{gather}
\end{definition}

\begin{figure*}[ht]
\centering
\centerline{\includegraphics[width=\textwidth]{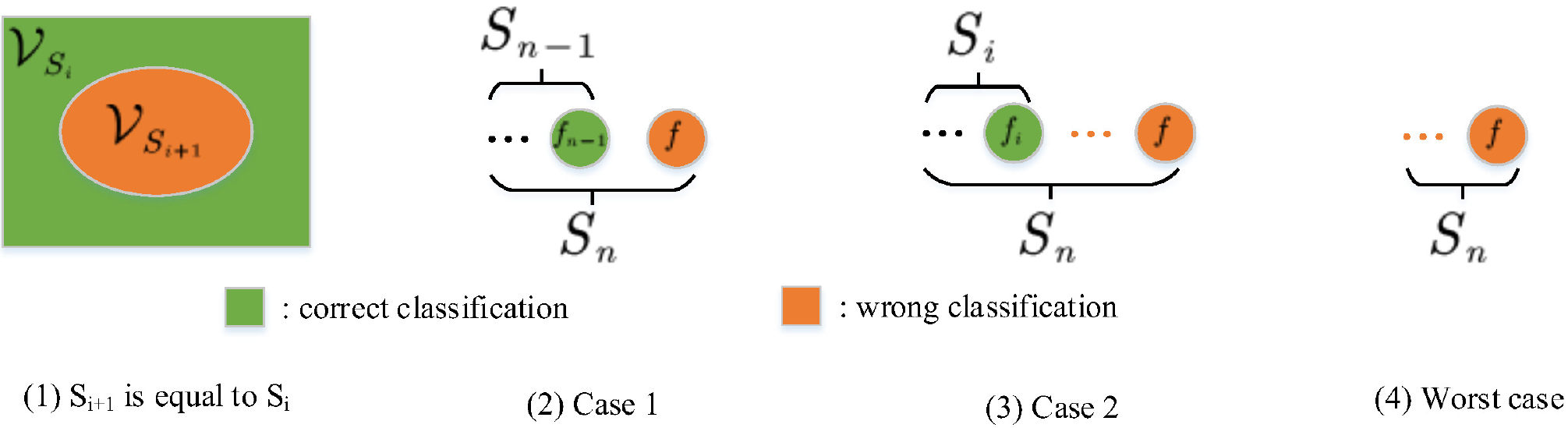}}
\caption{The working principle of the FA correction modules.}
\label{FA_analysis}
\end{figure*}

\begin{proposition}
	\label{prop:effectiveness_of_middle_layer}
	The impact of adversarial examples on the intermediate layer is less than that on the last layer.
\end{proposition}
\begin{proof}
	As shown in Fig.~\ref{FA_analysis}-(1), when the classification sequence of the top $i$ auxiliary classifiers is equal, i.e. $S_{i+1}[1\cdots i]=S_i$, the classification sequence space $\mathcal{V}_{S_{i+1}}$ is constrained in $\mathcal{V}_{S_i}$, i.e.,
	\begin{align}
	\label{equation:relationship_of_classification_sequence_space}
	\mathcal{V}_{S_{i+1}}\subset \mathcal{V}_{S_i},\quad if\quad S_{i+1}[1\cdots i]=S_i.
	\end{align}
	For an instance $(x,y)\sim\mathcal{D}$, its perturbation version $\hat{x}$ (i.e., an adversarial example) mislead the classifier $f$, i.e.,
	\begin{align}
	\hat{x}\notin \mathcal{V}_{S_n},\quad if\quad S_n[n]=y.
	\end{align}
	%
	
	According to Eq.~\ref{equation:relationship_of_classification_sequence_space}, because the space $\mathcal{V}_{S_i}$ ($i<n$) includes $\mathcal{V}_{S_n}$ when $S_n[1\cdots i]=S_i$, the adversarial example $\hat{x}$ may be located in the space $\mathcal{V}_{S_i}$. Therefore, if $S[i]=y$, the auxiliary classifier $f_i$ can correctly predict the adversarial example $\hat{x}$. The following two cases will discuss that the auxiliary classifier $f_i$ ($i<n$) can correctly predict the adversarial example $\hat{x}$:
	\begin{itemize}
		\item Case 1: as shown in Fig.~\ref{FA_analysis}-(2), the auxiliary classifier $f_{n-1}$ can correctly predict the adversarial example $\hat{x}$ when
		\begin{gather}
		\hat{x}\in \mathcal{V}_{S_{n-1}}\land S_{n-1}[n-1]=y.
		\end{gather}
		\item Case 2: Similarly, as shown in Fig.~\ref{FA_analysis}-(3), the auxiliary classifier $f_i$ ($i<n-1$) can correctly predict the adversarial example $\hat{x}$.
	\end{itemize}
	These two cases occur because the disturbance against the original clean example is less and as close to the original example as possible. Hence, the generated adversarial example only misleads the DNN classifier but is still correctly predicted by the auxiliary classifier. Therefore, the adversarial example mainly makes the last layer output the wrong decision, but the intermediate layers still extract effective features for the original category. \qed
\end{proof}

The Proposition~\ref{prop:effectiveness_of_middle_layer} verifies that the intermediate layer retains effective features for the original category against adversarial examples. Hence, embedding auxiliary classifiers into the DNN classifier to conduct the FA correction module, which is defined in Definition~\ref{FAmodel}, can effectively correct the classification.

\begin{definition}[FA correction module]\label{FAmodel}
	The $i^{th}$ FA correction module, denoted as $\varphi_i(\cdot)$, can be constituted by the emsemble of the classifier $f$ and the $i^{th}$ auxiliary classifier $f_i$, which can be fine-tuned by
	\begin{align}
	\label{equation:objective_function_of_FA_correction_model}
	\underset{\theta ^{\varphi_i}}{\min}\mathbb{E} _{\left( x,y \right) \sim \mathcal{D}}\left[ L_{CE}\left( \varphi_i\left( f(x)\oplus f_i\left(l_i\left( x \right)\right) \right) ,y \right) \right]
	\end{align}
	where $\oplus$ is the concatenation operation.
\end{definition}

Obviously, multiple FA correction modules are formed by utilizing different intermediate layers of the DNN. Besides, these modules can be collaboratively implemented in our FACM model to further correct the classification of the DNN classifier. In addition, these correction modules will not affect the classification accuracy on clean examples, which is verified in Section~\ref{sec:the_effect_of_FA_correction_modules} in detail.

\subsection{The Conditional Matching Prediction Distribution (CMPD) Correction Module}
\label{sec:CMPD}
To further improve the robustness of the DNN-implemented classifier in the scenarios (i.e., Fig.~\ref{FA_analysis}-(4)) that all auxiliary classifiers cannot correctly predict the adversarial examples, a conditional matching prediction distribution (CMPD) correction module is proposed, which can transform the adversarial sample $\hat{x}$ into $\hat{x}'$ to be correctly classified by the classifier $f$.

The CMPD correction module is the composite function of a conditional autoencoder and the DNN classifier $f$. The input of the conditional autoencoder consists of a clean example $x$ and the concatenation of the outputs of the top $i$ auxiliary classifiers. The output of the conditional autoencoder is the transformed example $x'$. Additionally, applying the outputs of different top $i$ auxiliary classifiers as the conditions will construct different conditional autoencoders, denoted as $g_i$, which is defined in Definition~\ref{def:CMPD}. Multiple CMPD correction modules can also be collaboratively embedded in our FACM model.

\begin{definition}[The $i^{th}$ conditional autoencoder $g_i(\cdot)$]
\label{def:CMPD}
The $i^{th}$ conditional autoencoder $g_i(\cdot)$ emerges an encoder $\varPhi(\cdot)$, a decoder $\varPsi(\cdot)$ and the outputs of the top $i$ auxiliary classifiers as the conditions, which can be represented as
\begin{gather}
g_i(x)=\varPsi(\varPhi(x| \xi_i(x))|\xi_i(x)) \label{equation:conditional_autoencoder}\\
\xi_i(x) = f_1(l_1(x))\oplus f_2(l_2(x))\oplus\cdots\oplus f_i(l_i(x)) \label{equation:conditions}
\end{gather}
where $\oplus$ denotes the concatenation operator. The $i^{th}$ conditional autoencoder $g_i(\cdot)$ is fine-tuned by
\begin{align}
\label{equation:objection_of_CMPD}
\underset{g_i\in \mathcal{H} _{g_i}}{\min}\mathbb{E} _{\left( x,y \right) \sim \mathcal{D}}\left[ L_{KL}\left( f\left( x \right), f\left( g_i\left( x \right) \right) \right) \right]
\end{align}
where $\mathcal{H}_{g_i}$ is the hypothesis space of the $i^{th}$ conditional autoencoder $g_i$.
\end{definition}

The $i^{th}$ CMPD correction module is represented as $f\circ g_i(\cdot)$. Additionally, the autoencoder $g_0$ means no output of the auxiliary classifier as the conditional input. In this scenario, the CMPD correction module $f\circ g_0(\cdot)$ will degenerate to the MPD~\cite{MPD} correction model $f\circ g(\cdot)$, where
\begin{align}
g_0(x)=g(x)=\varPsi(\varPhi(x))).
\end{align}
As shown in Fig.~\ref{framework_of_FACM}, all CMPD correction modules in our FACM model share parameters except for fully connected layers with conditional input. 

Proposition~\ref{prop:CMPD} verifies that the training of the conditional autoencoder $g_i$ ($i>0$) can achieve better convergence than that of the autoencoder $g_0$.
\begin{proposition}
\label{prop:CMPD}
With $\xi_i(x)$ as the conditional input, the conditional autoencoder $g_i$ can simplify the complexity of the learning task in comparison with the autoencoder $g_0$.
\end{proposition}
\begin{proof}
For an instance $(x,y)\sim \mathcal{D}$, if the classification sequence of $x$ is $S_n$, the autoencoder $g_0$ needs to find the classification sequence $S_n$ from $m^n$ classification sequences to ensure that
\begin{gather}
(argmax(f_1\circ g_0(x)), argmax(f_2\circ g_0(x)), \nonumber\\
\cdots, argmax(f_n\circ g_0(x)))=S_n
\end{gather}
where $m$ is the number of categories in the classifier $f$ and $f_i\circ g_0$ is the composite function of the autoencoder $g_0$ and the $i^{th}$ auxiliary classifier $f_i$:
\begin{align}
\label{equation:composite_function_of_auxiliary_and_autoencoder}
f_i\circ g_0(x)=f_i(l_i(g_0(x))).
\end{align}
While, for the conditional autoencoder $g_i$, due to the known condition $\xi_i(x)$, the classification sequence $S_i$ is known. Hence, it only needs to find the sub-sequence $S_n[i+1\cdots n]$ from $m^{n-i}$ classification sequences to ensure that 
\begin{gather}
(argmax(f_{i+1}\circ g_i(x)), \cdots,\nonumber\\
argmax(f_n\circ g_i(x)))=S_n[i+1\cdots n]
\end{gather}
where $f_{i+1}\circ g_i$ is the composite function of the conditional autoencoder $g_i$ and the $(i+1)^{th}$ auxiliary classifier $f_{i+1}$ as shown in Eq.~\ref{equation:composite_function_of_auxiliary_and_autoencoder} by replacing $g_0$ with $g_i$.

According to the concept of the hypothesis space and growth function in computational learning theory \cite{computational-learning-theory}, the growth functions of the hypothesis space of $g_0$ and $g_i$ can be calculated as Eq.~\ref{equation:growth_function}, where the hypothesis space is the set of all possible mappings and the growth function represents the maximum number of possible results. 
\begin{align}
\label{equation:growth_function}
\left\{ \begin{array}{c}
\varPi _{\mathcal{H} _g}\left( M \right) =M^{m^n}\\
\varPi _{\mathcal{H} _{g_i}}\left( M \right) =M^{m^{n-i}}\\
\end{array} \right.
\end{align}
where $\mathcal{H}_g$ and $\mathcal{H}_{g_i}$ denote hypothesis space of the autoencoder $g$ and the conditional autoencoder $g_i$ respectively, $\varPi _{\mathcal{H} _g}$ and $\varPi _{\mathcal{H} _{g_i}}$ denote the growth function of $\mathcal{H}_g$ and $\mathcal{H}_{g_i}$ respectively, $M$ is the size of training dataset. Due to $\varPi _{\mathcal{H}_g}(M)\gg \varPi _{\mathcal{H}_{g_i}}(M)$, the solution space of the conditional autoencoder $g_i$ is less than that of the autoencoder $g$. Therefore, the conditional autoencoder $g_i$ converges faster than the autoencoder $g$.

According to Theorem 12.2~\cite{computational-learning-theory} that utilizes the growth function to estimate the relationship between the empirical error $E(\cdot)$ and the generalization error $\hat{E}(\cdot)$, for any $M$, $0<\upsilon<1$, $g\in \mathcal{H}_g$ and $g_i\in \mathcal{H}_{g_i}$, we have
\begin{align}
\left\{ \begin{array}{c}
P( |E( g ) -\hat{E}( g ) |>\upsilon ) \leq 4\varPi _{\mathcal{H} _g}( 2M ) \exp ( -\frac{M\upsilon ^2}{8} )\\
P( |E( g_i ) -\hat{E}( g_i ) |>\upsilon ) \leq 4\varPi _{\mathcal{H} _{g_i}}( 2M ) \exp ( -\frac{M\upsilon ^2}{8} )\\
\end{array} \right.
\end{align}
where $P(|E(g)-\hat{E}(g)|>\upsilon)$ and $P(|E(g_i)-\hat{E}(g_i)|>\upsilon)$ respectively denote the probabilities of the autoencoder $g$ and the conditional autoencoder $g_i$ which do not converge to the expectation error $\upsilon$. Due to $4\varPi _{\mathcal{H} _{g_i}}\left( 2M \right) \exp \left( -\frac{M\upsilon ^2}{8} \right) \ll 4\varPi _{\mathcal{H} _g}\left( 2M \right) \exp \left( -\frac{M\upsilon ^2}{8} \right)$, the possible value range of $P(|E(g_i)-\hat{E}(g_i)|>\upsilon)$ are smaller, i.e., $g_i$ has a smaller probability of non-convergence. Therefore, the conditional autoencoder converges more stable than the autoencoder. \qed
\end{proof}


Note that, the CMPD correction module $f\circ g_j$ will achieve poor performance to restore the classification sequence $\hat{S}_n$ of adversarial example to the sequence $S_n$ of its natural version when the condition of $f\circ g_j$, i.e. $\hat{S}_n[1\cdots j]$, is not equal to $S_n[1\cdots j]$.
%
%
\begin{algorithm}[tb]
\caption{Decision model training}
\label{alg:decision_model_training}
\begin{algorithmic}
\Require
PGD adversarial training for $T$ epochs, given some radius set $\mathcal{E}$, adversarial step size $\mathcal{A}$, $N$ PGD steps and a dataset of size $M$.
\Ensure
Decision model $h$ and its parameters $\theta_h$.
\State $\mathcal{D}=\{(x_i,y_i)|i\leq M\}$
\For {$\epsilon, \alpha$ {\bfseries in} $\mathcal{E},\mathcal{A}$}
\For {$i=1$ {\bfseries to} $M$}
\State //Perform PGD adversarial attack
\State $\delta=0$ //or randomly initialized
\For {$j=1$ {\bfseries to} $N$}
\State $\delta=\delta+\alpha\cdot sign(\nabla_\delta\sum_c^{\mathbb{C}} L_{CE}(c(x_i+\delta), y_i))$
\State $\delta=\max(\min(\delta,\epsilon),-\epsilon)$
\EndFor
\State $\mathcal{D}=\mathcal{D}\cup\{(x_i+\delta,y_i)\}$
\EndFor
\EndFor
\For {$t=1$ {\bfseries to} $T$}
\For {$(x,y)$ {\bfseries in} $\mathcal{D}$}
\State // Update model parameters
\State $\theta _h=\theta _h-\nabla_{\theta_h} L_{MFL}(x,y)$
\EndFor
\EndFor
\end{algorithmic}
\end{algorithm}

\subsection{The Decision Module}
\label{sec:decision_model}
According to two cases analyzed in the proof of Proposition~\ref{prop:effectiveness_of_middle_layer} and the above note, different adversarial samples may need different correction modules to correct the outputs of the DNN classifier. Hence, multiple correction modules should be collaboratively implemented in our FACM model. In addition, the diversity of correction modules exists against adversarial examples, and the greater the attack strength, the more significant the diversity, which is defined as the \textit{Diversity Property} in our paper and demonstrated in Section~\ref{sec:the_effect_of_FA_correction_modules} and ~\ref{sec:the_effect_of_CMPD_correction_modules} in details. Due to the smaller adversarial subspace, the ensemble model has robustness against adversarial examples~\cite{SORS}. Hence, the \textit{Diversity Property} allows our correction modules to be integrated collaboratively to defend against adversarial examples.

To this end, a decision model is proposed to integrate the FA and CMPD modules implemented in the FACM model and determine the weights of the correction modules against an adversarial sample, thereby further enhancing the robustness of the DNN classifier against adversarial examples. Note that, different from the resource-consuming and classifier retraining of the model ensemble methods~\cite{Managing-diversity, Ensemble-adversarial-training} and Super-network~\cite{SORS}, our method only needs to fine-tune several fully-connected layers and three convolutional layers constructed encoder and decoder for achieving a diverse correction model set $\mathbb{C}$ with the size of $2n$ as shown in Eq.~\ref{equation:corection_model_set}.

Due to the adversarial examples making the deep neural network take an unusual activation path \cite{NIC}, in our model the outputs of all auxiliary classifiers will be used as the input of the decision module to decide which correction module in the set $\mathbb{C}$ used to predict the adversarial example. To train the decision module $h(\cdot)$ with an instance $(x,y)$, the input and ground truth of the decision module is denoted as $X^h(x)$ and $Y^h(x,y)$, respectively, where $X^h(x)$ indicates the concatenation of outputs of all auxiliary classifiers:
\begin{align}
\label{equation:input_annotation_function_of_h}
X^h(x)=\xi_n(x),
\end{align}
$Y^h(x,y)$ is a 0-1 vector used to indicate whether each correction module in the set $\mathbb{C}$ (i.e., each element in $[f(x), \varphi_1(f(x)\oplus f_1(l_1(x))),\cdots,\varphi_{n-1}(f(x)\oplus f_{n-1}(l_{n-1}(x))),f\circ g_0(x),\cdots,f\circ g_{n-1}(x)]$) is classified correctly or not and can be represented as
\begin{gather}
Y^h(x,y)=[sign(f(x), y), sign(\varphi_1(f(x)\oplus f_1(l_1(x))), y), \nonumber\\
\cdots, sign(\varphi_{n-1}(f(x)\oplus f_{n-1}(l_{n-1}(x))), y), \nonumber\\
sign(f\circ g_0(x), y), \cdots, f\circ g_{n-1}(x)] \label{equation:output_annotation_function_of_h}\\
sign\left( \boldsymbol{y} ,y \right) =\left\{ \begin{array}{c}
1, if\,\,argmax \left( \boldsymbol{y} \right) =y\\
0, if\,\,argmax \left( \boldsymbol{y} \right) \ne y\\
\end{array} \right..\label{equation:sign_function}
\end{gather}

Algorithm~\ref{alg:decision_model_training} indicates the training flow of the decision module. As shown in Figs.~\ref{FIG1} and \ref{FIG2}, because the accuracy of different correction modules is different in predicting the adversarial examples, the quantity marked with 1 in different modules is different. This will lead to the imbalance of training data for the decision module $h$. Hence, the multi-label focal loss is used as the loss function for the decision modules:
\begin{gather}
L_{MFL}(x,y)=\sum_{i}[-Y^h(x,y)_i\cdot( 1-Sigmoid( h( X^h(x) ) ) _i ) ^{\gamma}\nonumber\\
\cdot\log ( Sigmoid( h( X^h(x) ) ) _i )]\label{equation:multilabel_focal_loss}
\end{gather}
where $Sigmoid(\cdot)$ is the sigmoid function, $\gamma$ is an adjustable factor, $Sigmoid\left( h\left( X^h(x) \right) \right) _i$ and $Y^h(x,y)_i$ are the $i^{th}$ element in these vectors.

\subsection{The Operation Process of the FACM Model}
Taking an instance $(x,y)$ as an example, the weight vector $\boldsymbol{\omega}$ of all correction modules in the FACM model is determined by the decision module, which is represented as
\begin{align}
\label{equation:selection-weight-vector}
\boldsymbol{\omega} = \frac{Sigmoid\left( h\left( x \right) \right)}{\left| Sigmoid\left( h\left( x \right) \right) \right|_1}
\end{align}
where $Sigmoid(\cdot)$ is the sigmoid function. 

Then, according to the weight vector, $\boldsymbol{\omega}$, $\tau$ correction modules (both FA and CMPD) are randomly selected from FA and CMPD correction module set $\mathbb{C}$, which also includes the DNN-implemented classifier $f$.

Finally, the average output of $\tau$ selected correction modules is taken as the final prediction of the instance $x$:
\begin{gather}
F\left( x \right) =\frac{\sum\nolimits_{i=1}^{\tau}{c_i\left( x \right)}}{\tau},c_i\sim PN\left( \mathbb{C} :\boldsymbol{\omega} \right) \label{equation:FACM_prediction}\\
\mathbb{C} =\left\{ f, \varphi_1, \varphi_2, \cdots, \varphi_{n-1}, f\circ g_0, f\circ g_1, \cdots, f\circ g_{n-1} \right\} \label{equation:corection_model_set}\\
\left| \mathbb{C}\right|=2n
\end{gather}
where $F(x)$ denotes the output of the FACM model to the input $x$, $c_i$ is a selected correction model in $\mathbb{C}$, $PN(\mathbb{C}:\boldsymbol{\omega})$ represents a multinomial distribution where the correction model set $\mathbb{C}$ satisfies probability vector $\boldsymbol{\omega}$.

In our FACM model, the diversity of the correction modules and the randomness of the decision module can improve the robustness of the DNN-implemented classifier against adversarial attacks, especially optimization-based white-box attacks, and query-based black-box attacks. The adversarial examples generated by the optimization-based white-box attacks are model-specific, which have poor transferability~\cite{Delving-into-transferable}. Hence, our FACM model can correctly predict the adversarial examples generated by these attacks. Additionally, the diversity of the correction modules and the randomness of the decision module can lead to inaccurate gradient estimation of the query-based black-box attacks, thereby reducing the negative effect of the adversarial examples generated by this kind of attack.

\section{Experiments and Results}
\label{sec:experiments}
In this section, the experiment is conducted to validate the effectiveness of the proposed FACM model in enhancing the robustness of the DNN-implemented classifier against optimization-based white-box attacks and query-based black-box attacks. All experiments are run on a single machine with four GeForce RTX 2080tis using Pytorch.

\subsection{Experimental Setting}
\textbf{Datasets and the architecture of the classifier.} Our experiments are conducted on benchmark adversarial learning datasets, including MNIST~\cite{MNIST}, CIFAR10~\cite{CIFAR}, and CIFAR100~\cite{CIFAR} datasets. For the MNIST dataset, the algorithms with the model architecture MNISTNet~\cite{MNISTNet} are evaluated, where MNISTNet includes 4 convolutional layers and 3 fully connected layers. For both CIFAR10 and CIFAR100 datasets, the algorithms with the model architecture WRN-16-4~\cite{WRN} are evaluated, where WRN-16-4 includes 4 basic blocks.

\textbf{The architecture of the FACM model.} The architecture of the auxiliary classifier is a fully connected layer, in which the input is the output of a intermediate layer of the DNN and the output is the logit vector of categories. The architecture of the FA correction module is also a fully connected layer, in which the input is the concatenation of the output of an auxiliary classifier and the DNN classifier, and the output is the logit vector of categories. The architecture of the conditional autoencoder in the CMPD correction module consists of an encoder and a decoder, which are constructed with 3 convolutional layers, across different datasets. The decision module is a three-layer perceptron. Note that the intermediate layer of the DNN can be a single layer or a network block.

\textbf{Baselines.} Seven adversarial training methods are evaluated, i.e., Fast Adversarial Training (Fast-AT)~\cite{Fast-AT}, You Only Propagate Once (YOPO)~\cite{YOPO}, Adversarial Training with Hypersphere Embedding (ATHE)~\cite{ATHE}, Fair Robust Learning (FRL)~\cite{FRL}, Friendly Adversarial Training (FAT)~\cite{FAT}, TRADES~\cite{TRADES} and Adversarial Training with Transferable Adversarial examples (ATTA)~\cite{ATTA}. Two channel-wise activation suppressing methods are selected, i.e., Channel-wise Activation Suppressing (CAS)~\cite{CAS} and Channel-wise Importance-based Feature Selection (CIFS)~\cite{CIFS}. Two randomization methods without retraining the classifier are compared, i.e., resize and padding (RP)~\cite{Mitigating-adversarial-effects} and random smoothing (RS)~\cite{RS}. All baselines use the default setting except for RS with $\sigma=0.02$.

\textbf{Attacks.} Four iterative-based white-box attacks including Fast Gradient Sign Method (FGSM)~\cite{FGSM}, Projected Gradient Descent (PGD)~\cite{PGD}, Momentum Iterative FGSM (MIFGSM or MI)~\cite{MIFGSM} and AutoAttack (AA)~\cite{AutoAttack} are evaluated, respectively. Three optimization-based white-box attacks including Carlini Wagner $L_{\infty}$ (CW$_{\infty}$), $L_2$ (CW$_2$)~\cite{CW} and DeepFool $L_{2}$ (DF$_2$)~\cite{DeepFool}) and two black-box attacks including are Square (Sq)~\cite{Square} and NATTACK (NA)~\cite{NATTACK} are evaluated as well.

\textbf{Metric.} In the evaluation, the effectiveness of our FACM model is evaluated in terms of classification accuracy, the attack time (i.e., the time required for the attacks), and inference time (i.e., the time to the classification) on the test set. For the white-box attacks, all examples in the test set are used to evaluate the accuracy of each method. For black-box attacks, the first 1000 examples in the test set are used in the evaluation for Square, while the first 200 examples are used for NATTACK. Note that the attack time and inference time are the total calculation time on the test set, rather than the average calculation time of an example. Additionally, the average classification accuracy on the specified attacks (i.e., three optimization-based white-box attacks~\cite{DeepFool, CW} and two query-based black-box attacks~\cite{Square,NATTACK}) is considered to evaluate the effectiveness of our FACM model. \textbf{Avg.} in each table represents the average classification accuracy of the specified attacks.


To explain the \textit{Diversity Property} of the FA and CMPD correction modules in the FACM model against adversarial examples, a difference metric, denoted as $\zeta \left( c _1,c _2, \mathcal{S} \right)$, between two correction models $c_1$ and $c_2$ on the test set $\mathcal{S}$ is defined as
\begin{align}
\label{equation:difference}
\zeta \left( c _1,c _2, \mathcal{S} \right) =\frac{\left\| \boldsymbol{v}_{\mathcal{S}}^{c _1}|\boldsymbol{v}_{\mathcal{S}}^{c _2} \right\| _1-\left\| \boldsymbol{v}_{\mathcal{S}}^{c _1}\&\boldsymbol{v}_{\mathcal{S}}^{c _2} \right\| _1}{\left\| \boldsymbol{v}_{\mathcal{S}}^{c _1}|\boldsymbol{v}_{\mathcal{S}}^{c _2} \right\|}
\end{align}
where $c_1$ and $c_2$ are two correction models (including FA and CMPD correction modules), $\mathcal{S}$ is the test set, $\boldsymbol{v}_{\mathcal{S}}^{c_1}$ denotes 0-1 vector in which each element represents whether the correction module $c_1$ correctly predicts a test example in $\mathcal{S}$, $|$ and $\&$ respectively represent OR and AND operations, $\left\| \cdot \right\| _1$ denotes the 1 norm.

\textbf{Hyperparameters of the FACM model.} In the implementation of the FACM model, both the FA and the CMPD correction modules are fine-tuned by stochastic gradient descent (SGD) optimizer with a learning rate of 0.0005 for 10 epochs and 0.001 for 30 epochs, respectively. For the naturally trained classifier, the FA and CMPD correction modules are fine-tuned by Eqs.~\ref{equation:objective_function_of_auxiliary_classifier} and \ref{equation:objective_function_of_FA_correction_model} and Eq.~\ref{equation:objection_of_CMPD}, respectively. For the TRADES adversarially trained classifier, the FA correction module is fine-tuned by Eqs.~\ref{equation:trades_objection_of_auxiliary_classifier} and \ref{equation:trades_objection_of_FA_correction_model}:
\begin{align}
\label{equation:trades_objection_of_auxiliary_classifier}
\underset{\theta^{f_i}}{\min}\mathbb{E} _{\left( x,y \right) \sim \mathcal{D}}\left[ \begin{array}{c}
L\left( f_i\left( l_i\left(x\right) \right) ,y \right)+\beta \cdot \underset{\delta \in \mathcal{B}}{\max}L_{KL}(\\
f_i( l_i(x+\delta) ) ,f_i( l_i(x) ) )\\
\end{array} \right]
\end{align}
\begin{align}
\label{equation:trades_objection_of_FA_correction_model}
\underset{\theta^{\varphi_i}}{\min}\mathbb{E} _{\left( x,y \right) \sim \mathcal{D}}\left[ \begin{array}{c}
L\left( \varphi_i\left( f(x)\oplus f_i\left(l_i(x)\right) \right) ,y \right)+\beta \cdot \underset{\delta \in \mathcal{B}}{\max}\\
L_{KL}( \varphi_i( f(x+\delta)\oplus f_i\left(l_i(x+\delta)\right) ), \\
\varphi_i( f(x)\oplus f_i\left(l_i(x)\right) ) )\\
\end{array} \right].
\end{align}
For adversarially trained classifier, the CMPD module is still fine-tuned by Eq.~\ref{equation:objection_of_CMPD} using clean examples. Finally, the decision module is trained by Algorithm~\ref{alg:decision_model_training} with the SGD optimizer with a learning rate of 0.1 for 20 epochs. The hyperparameters of the algorithm are set as $T=10, \mathcal{E}=[2/255,4/255,8/255], \mathcal{A}=[0.3/255,0.4/255,0.9/255]$ on CIFAR and $T=20, \mathcal{E}=[0.1,0.2,0.3], \mathcal{A}=[0.01,0.02,0.03]$ on MNIST. The parameter $\gamma$ in Eq.~\ref{equation:multilabel_focal_loss} is set as 2. A multistep learning rate scheduler is used in each training phase in our FACM model to decay the learning rate by 0.1 at 1/4 and 3/4 of the total epoch.

\begin{table}
\centering \caption{The hyperparameters of each attack on CIFAR/MNIST.}
\renewcommand\arraystretch{1.0}
\footnotesize
\begin{center}
\begin{tabular}{ccccc}
\hline
Attack   & $\epsilon$  & $\alpha$       & $T$         & $\eta$      \\ \hline
FGSM     & $\frac{8}{255}/0.3$ & -            & -         & -         \\
PGD      & $\frac{8}{255}/0.3$ & $\frac{0.8}{255}/0.03$ & 20/40     & -         \\
MIFGSM   & $\frac{8}{255}/0.3$ & $\frac{2}{255}/0.1$    & 5/5       & -         \\
AA       & $\frac{8}{255}/0.3$ & -            & -         & -         \\
CW$_\infty$    & $\frac{8}{255}/0.3$ & -            & 10/50     & 0.01/0.01 \\
CW$_2$      & $0.005/10$  & -            & 10/50     & 0.2/1.0   \\
DeepFool$_2$ & -         & -            & 50/50     & -         \\
Square   & $0.05/0.3$  & -            & 5000/5000 & -         \\
NATTACK  & $\frac{8}{255}/0.3$ & -            & 500/500   & -         \\ \hline
\end{tabular}
\end{center}
\label{tab:hyperparameters_of_attacks}
\vskip -0.2in
\end{table}

\textbf{Hyperparameters of the attacks.} The aforementioned white-box attacks and black-box attacks are implemented in \cite{TorchAttacks}. The hyperparameters of each attack are introduced in Table~\ref{tab:hyperparameters_of_attacks} where $\epsilon, \alpha, T, \eta$ denote the attack strength, the step size, the number of steps, and the learning rate, respectively. Besides, the overshoot is 0.02 on all datasets for DeepFool$_2$. The population is 100 on all datasets for NATTACK.

In addition, the evaluations consider two settings on the white-box attacks: grey-box and white-box settings, respectively. \textbf{In the grey-box setting (denoted as FACM-grey)}, the adversary only has full knowledge of the DNN-implemented classifier and does not have any information about our FACM model. \textbf{In the white-box setting (denoted as FACM-white)}, the adversary has the full knowledge of both the DNN-implemented classified and our FACM model. 

%


\begin{figure*}[t]
\begin{center}
\subfloat[The test accuracy curve of different FA correction modules $\{\varphi_i\}$ and the classifier $f$ on MNIST, CIFAR10 and CIFAR100 when the naturally trained classifier $f$ is attacked by FGSM with different $\epsilon$.]{\includegraphics[width=0.48\textwidth]{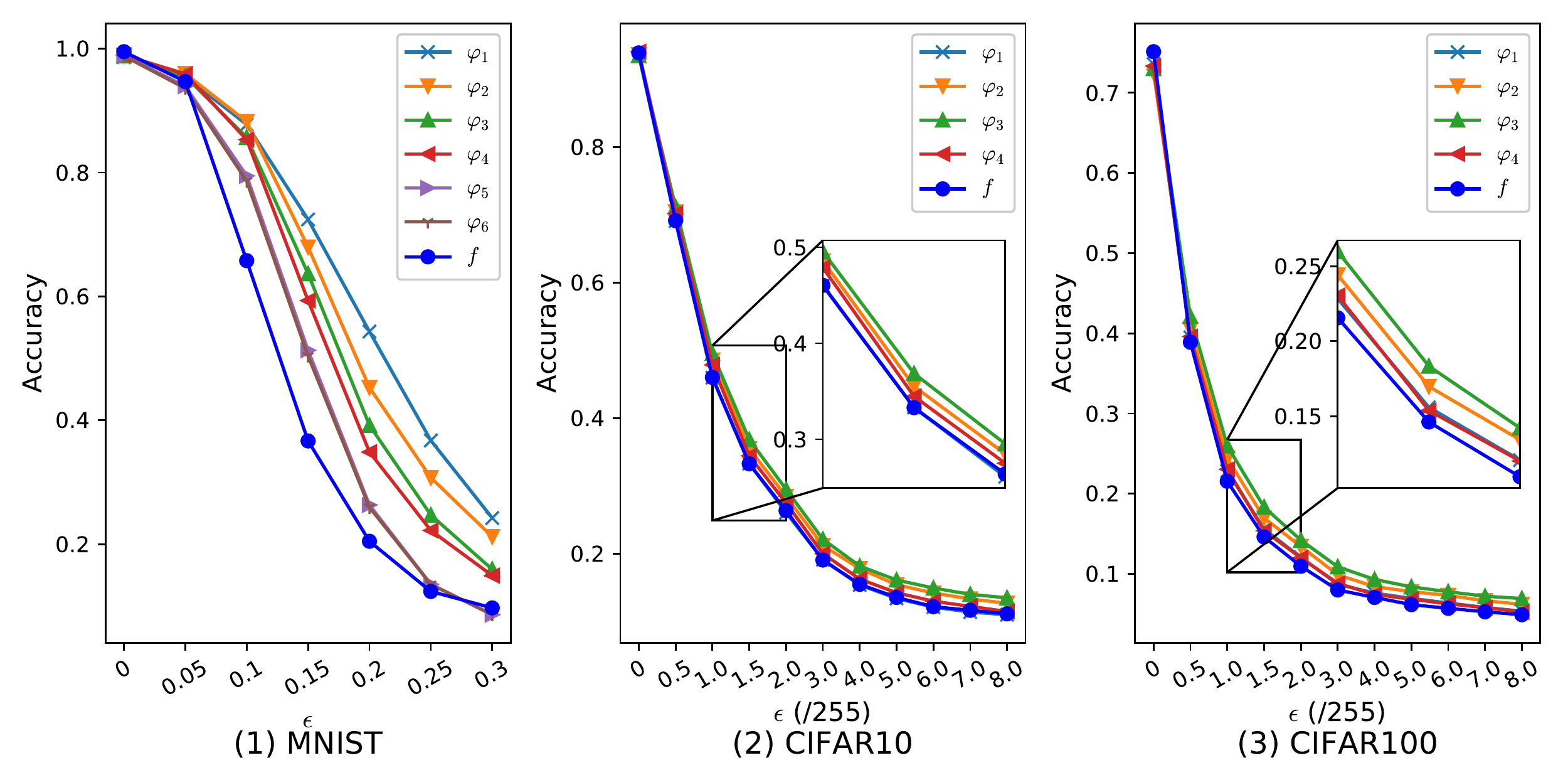}
\label{FIG1}
}
\hspace{1.0mm}
\subfloat[The test accuracy curve of different CMPD correction modules $\{f\circ g_i\}$ on MNIST, CIFAR10 and CIFAR100 when the naturally trained classifier $f$ is attacked by FGSM with different $\epsilon$.]{\includegraphics[width=0.48\textwidth]{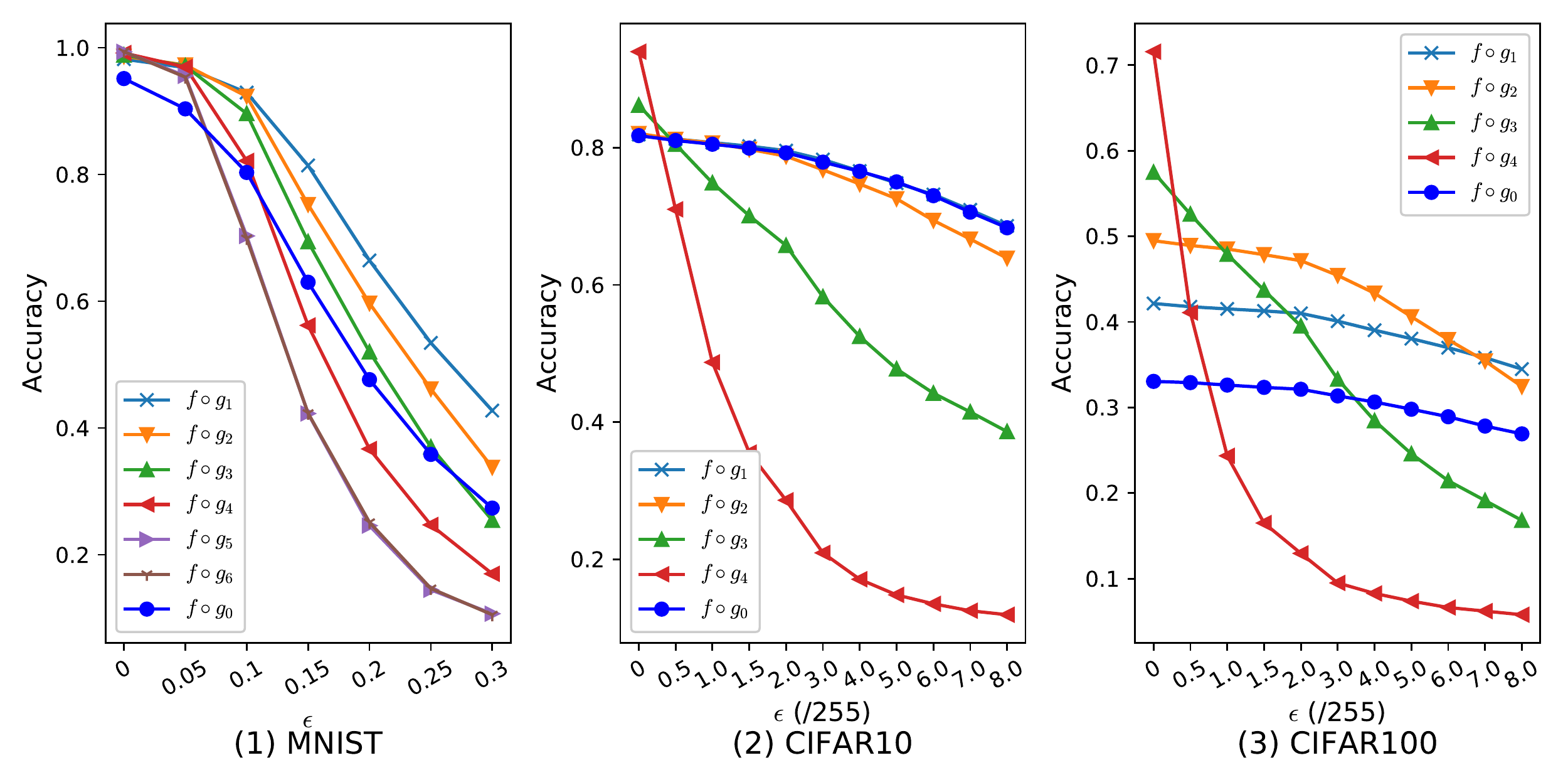}
\label{FIG2}
}
\end{center}
\caption{The test accuracy curve.}
\end{figure*}


\begin{figure*}[ht]
\centering
\subfloat[The difference (which is calculated by Eq.~\ref{equation:difference}) matrix heatmap between different FA correction models $\{\varphi_i\}$ and the classifier $f$ on MNIST, CIFAR10 and CIFAR100 when the naturally trained classifier $f$ is attacked by FGSM with different $\epsilon$.]{\includegraphics[width=0.48\textwidth]{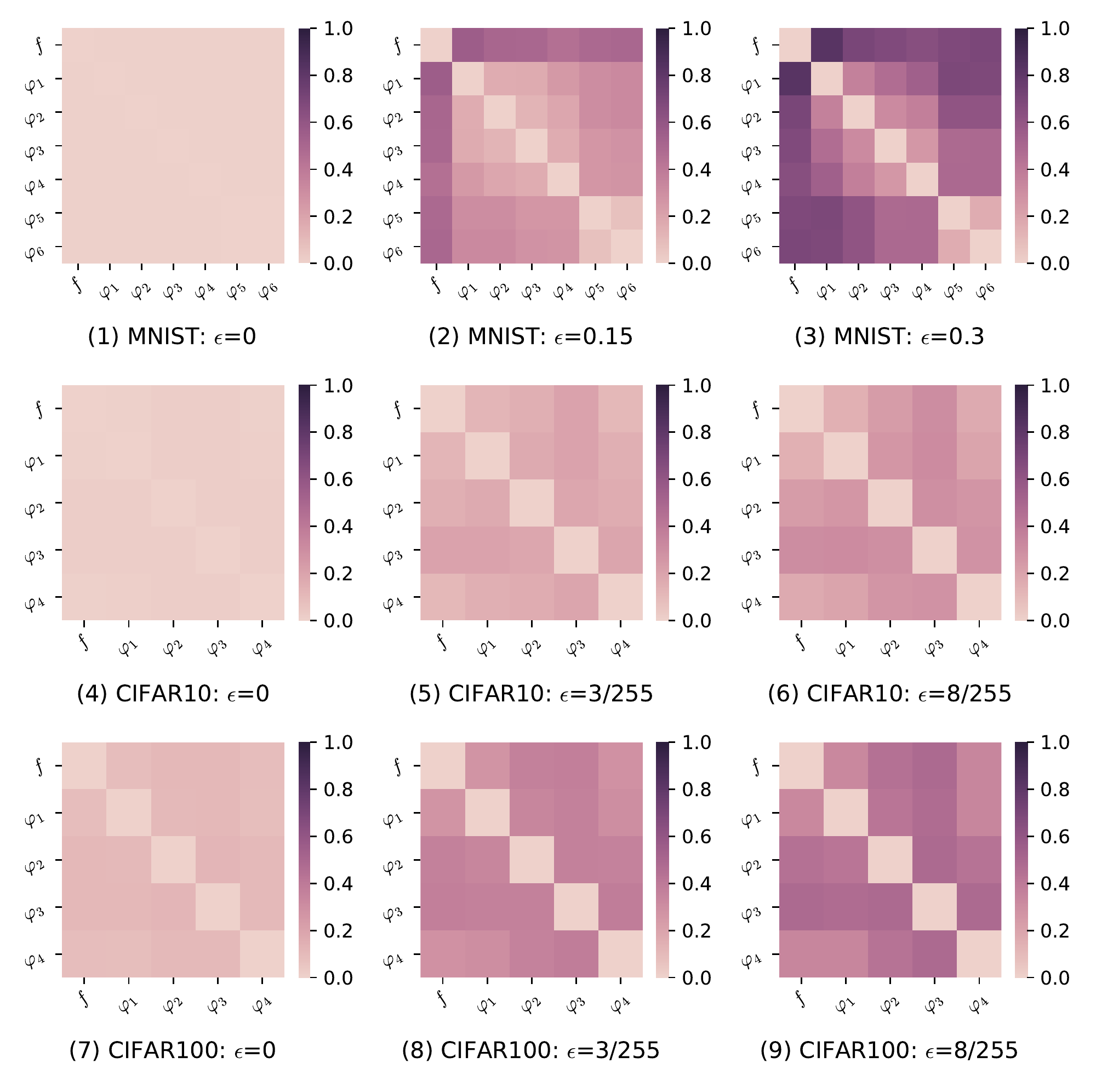}
\label{FIG5}
}
\hspace{1.0mm}
\subfloat[The difference (which is calculated by Eq.~\ref{equation:difference}) matrix heatmap between different CMPD correction models $\{f\circ g_i\}$ on MNIST, CIFAR10 and CIFAR100 when the naturally trained classifier $f$ is attacked by FGSM with different $\epsilon$.]{\includegraphics[width=0.48\textwidth]{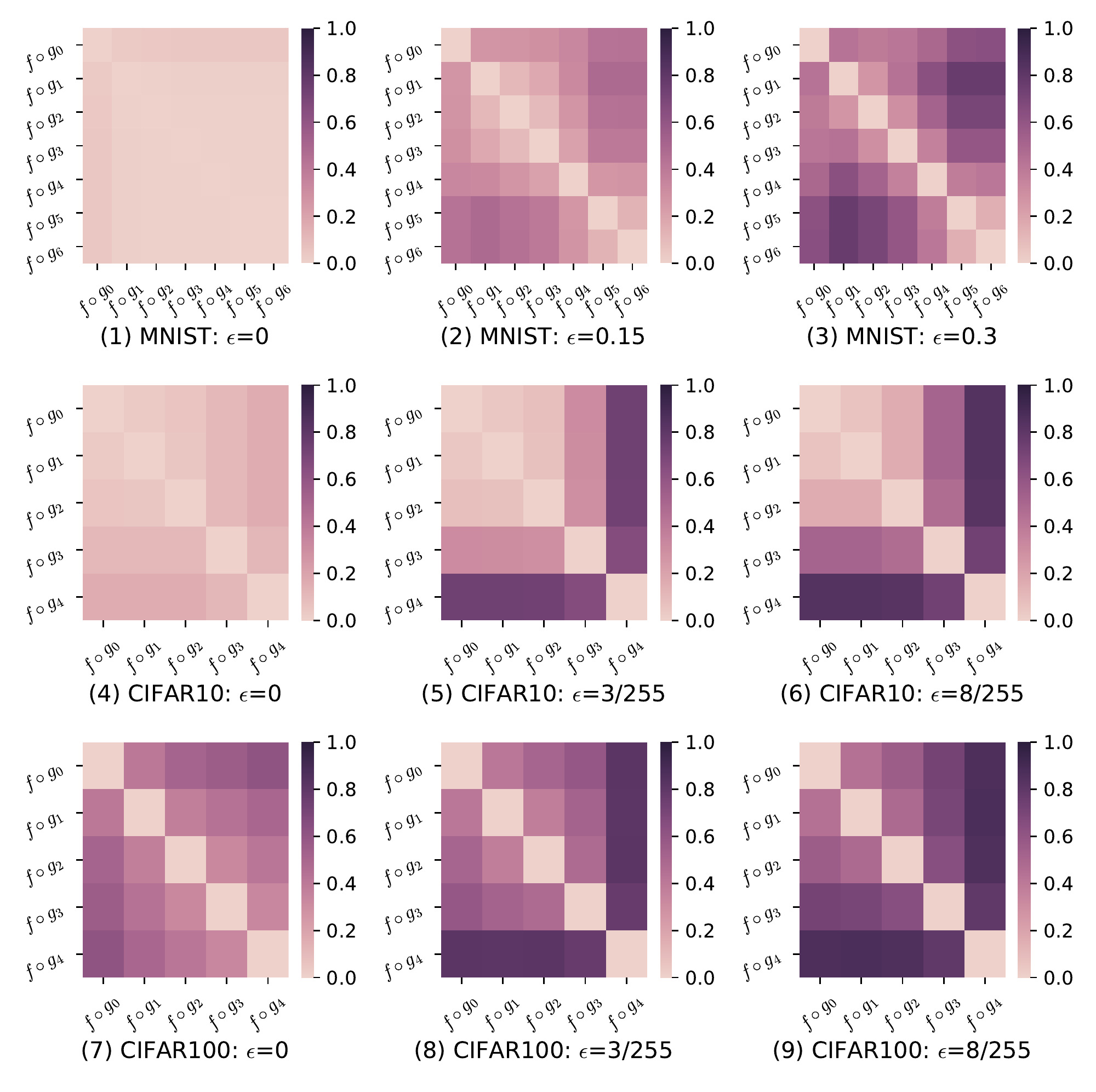}
\label{FIG9}
}
\caption{The difference matrix heatmap.}
\end{figure*}

\subsection{The Effect of the FA Correction Modules}
\label{sec:the_effect_of_FA_correction_modules}
To verify the effect of the \textit{Correction Property} of the FA correction modules, i.e., the intermediate layer's features can correct the output of the classifier during adversarial attacks, Fig.~\ref{FIG1} compares the test accuracy of the classifier $f$ with and without the FA correction modules $\{\varphi_i\}$ when the classifier $f$ attacked by FGSM (i.e., grey-box setting) in different attack strengths on MNIST, CIFAR10 and CIFAR100 datasets. As shown in Fig.~\ref{FIG1}, the blue curve represents the test accuracy of the classifier $f$ without FA correction modules, and the other curves represent that of the FA correction models. Fig.~\ref{FIG1} shows that as the attack strength $\epsilon$ increases, the test accuracy of the classifier $f$ without FA correction modules decreases faster and is lower than the test accuracy of the classifier $f$ with all types of FA correction modules. That means the FA correction modules have a positive effect on improving the test accuracy of the classifier against adversarial examples with various attack strengths.

In addition, Fig.~\ref{FIG5} verifies the \textit{Diversity Property} of the FA correction modules when the classifier $f$ suffers from adversarial attacks. Fig.~\ref{FIG5} shows the matrix heatmaps among different FA correction modules on three datasets with different attack strengths. The results demonstrate that as the attack strength gradually increases, the differences among various FA correction modules are increasingly apparent. Hence, the \textit{Diversity Property} exists among the FA correction modules.

Note that, Fig.~\ref{FIG1} and Fig.~\ref{FIG5} demonstrate the \textit{Correction Property} and the \textit{Diversity Property} of the FA correction modules against adversarial attacks under the grey-box setting. The corresponding evaluations of the FA correction modules against attacks under the white-box setting are shown in the Supplementary, which shows that, in the adversarial attacks under the white-box setting, although the FA correction modules can effectively correct the outputs of the classifier only on MNIST dataset, the \textit{Diversity Property} of the FA correction modules can also be demonstrated to exist in all MINST, CIFAR10, and CIFAR100 datasets.

\subsection{The Effect of the CMPD Correction Modules}
\label{sec:the_effect_of_CMPD_correction_modules}
To verify the effect of the \textit{Correction Property} of the CMPD correction modules when the classifier $f$ is attacked, the test accuracy of the CMPD correction modules is compared to that of the MPD correction model $f\circ g_0$ when the classifier $f$ is attacked by FGSM with different attack strengths on MNIST, CIFAR10, and CIFAR100, respectively. As shown in Fig.~\ref{FIG2}, the blue curve represents the test accuracy of the MPD correction module $f\circ g_0$ and the other curves represent the test accuracy of the various CMPD correction modules. Fig.~\ref{FIG2} shows that the CMPD correction module $f\circ g_i$ that uses the outputs of the shallow auxiliary classifiers (i.e., $i$ is small) as a condition achieves greater test accuracy than the MPD correction module $f\circ g_0$ on adversarial examples with different attack strengths in MNIST and CIFAR100 dataset, and achieves comparable test accuracy to the MPD correction module $f\circ g_0$ in CIFAR10 dataset. The results demonstrate that the CMPD correction modules can achieve great effectiveness in correcting the outputs of the classifier on adversarial examples.

In addition, Fig.~\ref{FIG9} verifies the \textit{Diversity Property} of the CMPD correction modules when the classifier $f$ is attacked. As shown in Fig.~\ref{FIG9}, the difference matrix heatmaps of the CMPD correction modules with different attack strengths show that, as the attack strength gradually increases, the differences among the CMPD correction modules significantly increase as well. That means the \textit{Diversity Property} exists among the various CMPD correction modules.

Note that, Fig.~\ref{FIG2} and Fig.~\ref{FIG9} demonstrate the \textit{Correction Property} and the \textit{Diversity Property} of the CMPD correction modules against adversarial attacks with the grey-box setting. The corresponding evaluations of the CMPD correction modules against attacks under the white-box setting shown in the Supplementary, which show that, in the adversarial attacks under the white-box setting, the \textit{Correction Property} and \textit{Diversity Property} of the CMPD correction modules are demonstrated to exist in all MINST, CIFAR10, and CIFAR100 datasets. That is, the CMPD correction modules can effectively correct the outputs of the classifier in the adversarial attacks with the white-box setting.

\begin{table*}[t]
\caption{Evaluation (test accuracy(\%)) of the naturally trained classifier with and without the FACM model against various attacks on CIFAR10/100 and MNIST datasets. F1, P1 represent FGSM($\epsilon$=2/255 on CIFAR10/100 and $\epsilon$=0.1 on MNIST), PGD($\epsilon$=2/255, $\alpha$=0.2/255, $T$=20 on CIFAR10/100 and $\epsilon$=0.1, $\alpha$=0.01, $T$=40 on MNIST), respectively. The attacks with gray background color represent the specified attacks.}
\renewcommand\arraystretch{1.0}
\footnotesize
\centering
\resizebox{1.0\textwidth}{!}{
\begin{tabular}{ccccccccccc}
\hline
\multicolumn{3}{c|}{}                                                         & \multicolumn{5}{c|}{White-box attacks}                                                                           & \multicolumn{2}{c|}{\multirow{2}{*}{\begin{tabular}[c]{@{}c@{}}Black-box\\ attacks\end{tabular}}} & \multirow{2}{*}{\textbf{}} \\ \cline{1-8}
\multicolumn{3}{c|}{}                                                         & \multicolumn{2}{c|}{iterative-based}      & \multicolumn{3}{c|}{optimization-based}                              & \multicolumn{2}{c|}{}                                                                             &                                \\ \cline{1-10}
                          & Method               & \multicolumn{1}{c|}{Clean} & FGSM           & \multicolumn{1}{c|}{PGD} & \cellcolor{lightgray}DF$_{2}$       & \cellcolor{lightgray}CW$_{2}$       & \cellcolor{lightgray} CW$_{\infty}$ & \cellcolor{lightgray}Square                                           & \cellcolor{lightgray}NATTACK                   & \cellcolor{lightgray}\textbf{Avg.}                               \\ \hline
\multirow{3}{*}{CIFAR10}  & Natural              & \textbf{94.41}             & 9.69           & 0                        & 3.96           & 9.59           & 0                                  & 0                                                & 1.5                                            & 3.01                           \\
                          & +FACM-white($\tau$=3) & 91.91                      & \textbf{26.14} & \textbf{0.61}            & \textbf{13.41} & \textbf{18.62} & \textbf{23.76}                     & \multirow{2}{*}{\textbf{70.3}}                   & \multirow{2}{*}{\textbf{80}}                   & \textbf{41.22}                 \\ \cline{2-8} \cline{11-11} 
                          & +FACM-grey($\tau$=3)  & 92.08                      & 48.76          & 11.32                    & 77.61          & 34.72          & 51.06                              &                                                  &                                                & 62.74                          \\ \hline
\multirow{3}{*}{CIFAR100} & Natural              & \textbf{75.84}             & 5.3            & 0                        & 5.99           & 4.06           & 0                                  & 0                                                & 0                                              & 2.01                           \\
                          & +FACM-white($\tau$=3) & 73.37                      & \textbf{17.63} & \textbf{9.78}            & \textbf{35.38} & \textbf{10.6}  & \textbf{20.91}                     & \multirow{2}{*}{\textbf{34.2}}                   & \multirow{2}{*}{\textbf{51}}                   & \textbf{30.42}                 \\ \cline{2-8} \cline{11-11} 
                          & +FACM-gray($\tau$=3)  & 72.93                      & 23.84          & 19                       & 54.88          & 13.1           & 35.87                              &                                                  &                                                & 37.81                          \\ \hline
\multirow{3}{*}{MNIST}    & Natural              & \textbf{99.53}             & 9.7            & 0                        & 1.6            & 2.7            & 0                                  & 0                                                & 32                                             & 7.26                           \\
                          & +FACM-white($\tau$=1) & 98.60                      & \textbf{13.89} & 0                        & \textbf{6.95}  & \textbf{75.39} & \textbf{50.76}                     & \multirow{2}{*}{\textbf{60.8}}                   & \multirow{2}{*}{\textbf{98}}                   & \textbf{58.38}                 \\ \cline{2-8} \cline{11-11} 
                          & +FACM-grey($\tau$=1)  & 98.68                      & 23.36          & 8.07                     & 85.37          & 61.37          & 79.63                              &                                                  &                                                & 77.03                          \\ \hline
\end{tabular}
}
\label{tab1}
\end{table*}

\subsection{The Effect of the FACM Model}
\label{sec:the_effect_of_the_FACM_model}
To verify that the FACM model can effectively improve the robustness of naturally trained classifiers against various adversarial attacks, especially optimization-based white-box attacks and query-based black-box attacks, we compare the classification accuracy of the naturally trained classifier with and without the FACM model on CIFAR10/100 and MNIST. As shown in Table~\ref{tab1}, our model can improve the classification accuracy against iterative-based white-box attacks under the white-box setting. Besides, the average improvement of the classification accuracy with our model is 38.21\% on CIFAR10, 28.41\% on CIFAR100, and 51.12\% on MNIST against the specified attacks, i.e. the optimization-based white-box attacks and query-based black-box attacks. The classification accuracy with our model on clean samples only decreases by 2.5\%, 2.47\%, and 0.93\% on CIFAR10/100 and MNIST, respectively. Under the grey-box setting, the FACM model can provide greater robustness for the naturally trained classifier against all types of attacks. Hence, in both white-box and black-box settings, the FACM model can significantly improve classification accuracy of the naturally trained classifiers against optimization-based white-box attacks and query-based black-box attacks.

\begin{table*}[t]
\caption{Evaluation (test accuracy(\%)) of different randomization methods against various attacks on CIFAR10. The attack strength of FGSM and PGD is set as $2/255$ for attacking the FACM based naturally trained classifier. The attacks with gray background color represent the specified attacks.}
\renewcommand\arraystretch{1.0}
\centering
\footnotesize
\resizebox{1.0\textwidth}{!}{
\begin{tabular}{cccccccccc}
\hline
        & Clean         & FGSM           & PGD           & \cellcolor{lightgray}CW$_2$            & \cellcolor{lightgray}CW$_\infty$         & \cellcolor{lightgray}DeepFool            & \cellcolor{lightgray}Square            & \cellcolor{lightgray}NATTACK            & \cellcolor{lightgray}\textbf{Avg.}           \\ \hline
Natural & \textbf{94.4} & 9.54           & 0             & 9.51           & 0              & 3.5            & 0.1           & 1.5           & 2.92          \\
RP~\cite{Mitigating-adversarial-effects}      & 93.52         & 20.38          & 0.87          & 16.44          & 81.83          & \textbf{91.27} & 70.1          & 76.5          & 67.23          \\
RS~\cite{RS}      & 89.4          & 9.9            & 0             & 11             & 65.7           & 79.6           & 26            & 51            & 46.66          \\
FACM-grey($\tau=3$)    & 92.08         & 49.61 & 25.62         & 34.72          & 51.06          & 77.61           & 70.3          & 80          & 62.74          \\
RP+FACM-grey($\tau=3$) & 91.68         & \textbf{64.6}          & \textbf{54.09}         & \textbf{37.91} & \textbf{83.52} & 89.47 & \textbf{77.8} & \textbf{87} & \textbf{75.14} \\
RS+FACM-grey($\tau=3$) & 87.6          & 51.8           & 34 & 32.3           & 73.8           & 83.2           & 66.5            & 74.5          & 66.06          \\ \hline
\end{tabular}
}
\label{tab:randomization_comparisons}
\end{table*}

\subsection{The Compatibility with the Existing Defense Methods}
\label{sec:comparison_with_the_existing_defense_methods}
This section will investigate the compatibility of our FACM model with the existing defense methods, including randomization methods and adversarial training methods.
\subsubsection{The Compatibility between the FACM Model and Randomization Methods}
\label{sec:compatibility_between_FACM_and_Randomization}
In this section, we study the compatibility between our FACM model and two randomization methods (i.e., RP~\cite{Mitigating-adversarial-effects} and RS~\cite{RS}) on the naturally trained classifier under the grey-box setting. Note that, unlike other randomization methods~\cite{SLQ,Adversarial-noise-layer,Random-self-ensemble,Randomized-diversification,SORS}, RP~\cite{Mitigating-adversarial-effects} and RS~\cite{RS} do not need to change the architecture of the classifier and retrain it like our FACM model. As shown in Table~\ref{tab:randomization_comparisons}, our FACM model has higher robustness on iterative-based white-box attacks and competitive performance on optimization-based white-box attacks and query-based black-box attacks. When combining our FACM model with these randomization methods, our model can significantly improve the robustness of RP~\cite{Mitigating-adversarial-effects} and RS~\cite{RS} against all attacks except for the competitive robustness on DeepFool. Hence, our FACM model can be compatible with the existing randomization methods.

\begin{figure*}[ht]
\centering
\centerline{\includegraphics[width=\textwidth]{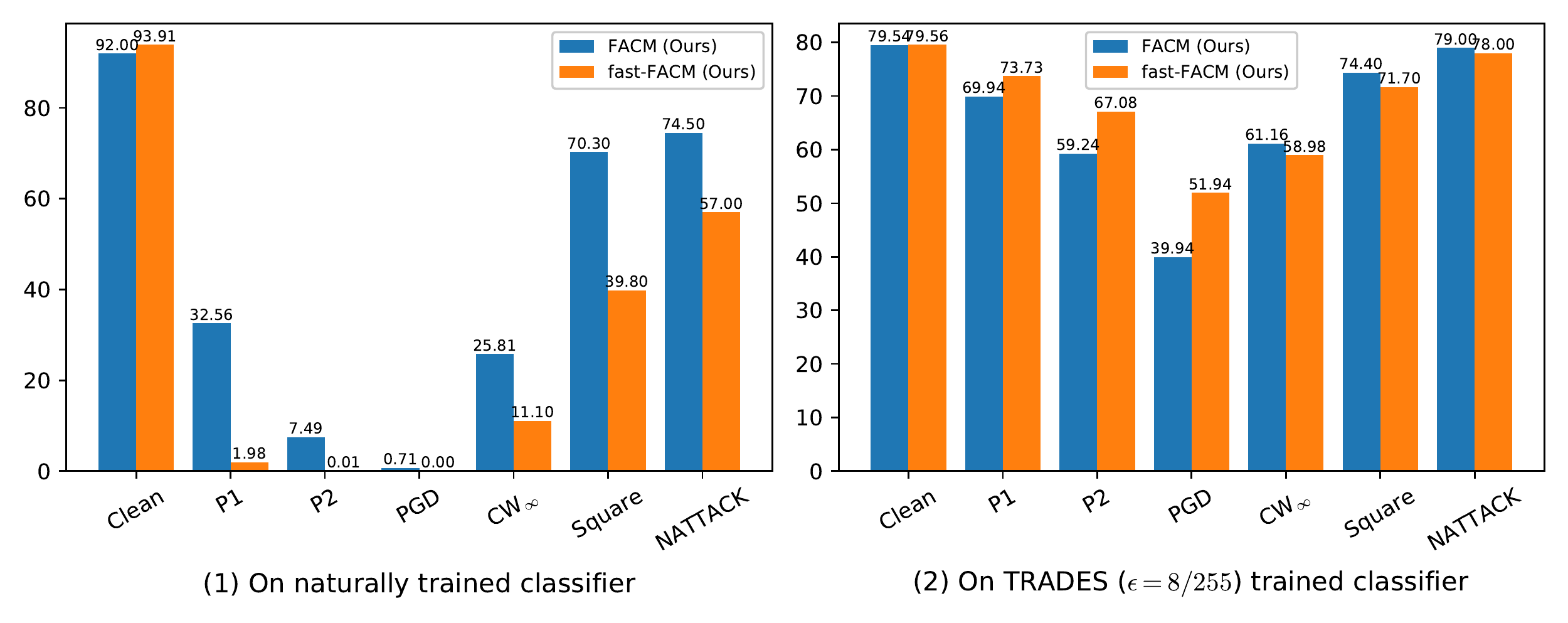}}
\caption{The test accuracy (\%) comparison between FACM and fast-FACM on naturally trained classifier and TRADES ($\epsilon=8/255$) trained classifier against various attacks on CIFAR10 dataset. P1, P2 represent PGD with the parameters $\epsilon=2/255$, $\alpha=0.2/255$, $T=20$ and $\epsilon=4/255$, $\alpha=0.4/255$, $T=20$, respectively.}
\label{FIG13}
\end{figure*}


\subsubsection{The Comparison between the FACM Model and fast-FACM model}
\label{sec:comparison_between_FACM_and_fastFACM}
In this section, we propose a more efficient model, named fast-FACM. In comparison with the FACM model, the fast-FACM model removes the CMPD modules. Hence, the model can achieve less construction complexity and faster inference speed, as shown in Table~\ref{tab:inference_time}. Then, we compare the effect of the FACM and fast-FACM models on naturally trained and adversarially trained classifiers, in which TRADES~\cite{TRADES} is chosen to train the classifier with the attack strengths $\epsilon=8/255$. As shown in Fig.~\ref{FIG13}, the FACM model has higher robustness against various attacks than fast-FACM on the naturally trained classifier. However, fast-FACM has higher robustness against iterative-based white-box attacks and approaching robustness against other attacks on the adversarially trained classifier. Besides, the fast-FACM model has less construction complexity and faster inference time, as shown in Table~\ref{tab:inference_time}. Therefore, the fast-FACM model is more compatible with adversarial training methods than the FACM model.

\begin{table*}[t]
\caption{Evaluation (test accuracy(\%)) of the adversarially trained classifier with and without the fast-FACM (fFACM) model against various attacks on CIFAR10/100 and MNIST datasets. The attacks with gray background color represent the specified attacks.}
\renewcommand\arraystretch{1.0}
\centering
\footnotesize
\resizebox{1.0\textwidth}{!}{
\begin{tabular}{ccccccccccccc}
\hline
                                                                     &                           & \multicolumn{1}{c|}{}      & \multicolumn{7}{c|}{White-box attacks}                                                                                                       & \multicolumn{2}{c|}{\multirow{2}{*}{\begin{tabular}[c]{@{}c@{}}Black-box\\ attacks\end{tabular}}} & \multirow{2}{*}{\textbf{}} \\ \cline{1-10}
                                                                     &                           & \multicolumn{1}{c|}{}      & \multicolumn{4}{c|}{iterative-based}                                & \multicolumn{3}{c|}{optimization-based}                            & \multicolumn{2}{c|}{}                                                                             &                                \\ \cline{1-12}
                                                                     & Method                    & \multicolumn{1}{c|}{Clean} & FGSM           & PGD            & MI             & \multicolumn{1}{c|}{AA} & \cellcolor{lightgray}DF$_{2}$       & \cellcolor{lightgray}CW$_{2}$       & \cellcolor{lightgray}CW$_{\infty}$ & \cellcolor{lightgray}Sq                                              & \cellcolor{lightgray}NA                         & \cellcolor{lightgray}\textbf{Avg.}                               \\ \hline
\multirow{3}{*}{\begin{tabular}[c]{@{}c@{}}CIFAR\\ 10\end{tabular}}  & TRADES \cite{TRADES}                   & \textbf{80.98}             & \textbf{55.8}  & \textbf{51.83} & \textbf{53.76} & \textbf{47.2}           & 0.59           & 25.2           & 48.54                              & 34.4                                            & 69                                              & 35.55                          \\
                                                                     & +fFACM-white($\tau=1$) & 79.42                      & 55.09          & 51.63          & 53.66          & 47.1                    & \textbf{33.66} & \textbf{27.96} & \textbf{61.78}                     & \multirow{2}{*}{\textbf{72.6}}                  & \multirow{2}{*}{\textbf{76}}                    & \textbf{54.4}                  \\ \cline{2-10} \cline{13-13} 
                                                                     & +fFACM-grey($\tau$=1)  & 79.47                      & 55.59          & 51.86          & 53.99          & -                       & 35.42          & 27.69          & 59.1                               &                                                 &                                                 & 54.16                          \\ \hline
\multirow{3}{*}{\begin{tabular}[c]{@{}c@{}}CIFAR\\ 100\end{tabular}} & TRADES \cite{TRADES}                   & \textbf{55.91}             & \textbf{29.23} & \textbf{26.91} & \textbf{27.82} & 21.7                    & 0.36           & 8.9            & 22.36                              & 13.3                                            & 39.5                                            & 16.88                          \\
                                                                     & +fFACM-white($\tau$=3) & 51.17                      & 27.62          & 26.33          & 26.83          & \textbf{22.0}           & \textbf{25.46} & \textbf{12.67} & \textbf{36.53}                     & \multirow{2}{*}{\textbf{46.7}}                  & \multirow{2}{*}{\textbf{51}}                    & \textbf{32.17}                 \\ \cline{2-10} \cline{13-13} 
                                                                     & +fFACM-grey($\tau$=3)  & 51.34                      & 28.81          & 27.11          & 27.91          & -                       & 32.99          & 13.99          & 38.17                              &                                                 &                                                 & 36.57                          \\ \hline
\multirow{3}{*}{MNIST}                                               & TRADES \cite{TRADES}                   & \textbf{99.48}             & \textbf{93.36} & 70.69          & 82.32          & 27.7                    & 4.43           & 68.17          & 96.19                              & 19.7                                            & 95.5                                            & 56.80                          \\
                                                                     & +fFACM-white($\tau$=2) & 98.62                      & 91.09          & \textbf{81.93} & \textbf{83.68} & \textbf{52.9}           & \textbf{40.79} & \textbf{96.37} & \textbf{97.03}                     & \multirow{2}{*}{\textbf{93.4}}                  & \multirow{2}{*}{\textbf{97.5}}                  & \textbf{85.02}                 \\ \cline{2-10} \cline{13-13} 
                                                                     & +fFACM-grey($\tau$=2)  & 98.69                      & 94.11          & 83.8           & 88.21          & -                       & 66.6           & 85.99          & 97.45                              &                                                 &                                                 & 88.19                          \\ \hline
\end{tabular}
}
\label{tab2}
\end{table*}

\subsubsection{The Effect of the fast-FACM Model on the Adversarially Trained Classifier}
\label{sec:fastFACM_on_the_adversarially_trained_DNN}
To verify the effect of the fast-FACM model on improving the robustness of the adversarially trained classifier against optimization-based white-box attacks and query-based black-box attacks, the classification accuracy of the TRADES trained model with and without the fast-FACM model is evaluated on CIFAR10/100 and MNIST, respectively. As shown in Table~\ref{tab2}, under the white-box setting, the fast-FACM model improves the average classification accuracy by 18.85\% on CIFAR10, 15.29\% on CIFAR100 and 28.22\% on MNIST, and keeps or slightly decreases the classification accuracy on clean examples and iterative-based white-box attacks on CIFAR10/100. Under the grey-box setting, the fast-FACM model can further improve the robustness of the adversarially trained classifier. The results demonstrate that the fast-FACM model can significantly improve the classification accuracy of the adversarially trained classifier against optimization-based white-box attacks and query-based black-box attacks.

\begin{table*}[t]
\caption{Evaluation (test accuracy(\%)) of different adversarial training methods on CIFAR10 against optimization-based white-box attacks and query-based black-box attacks. fFACM denotes fast-FACM.}
\renewcommand\arraystretch{1.0}
\centering
\footnotesize
\resizebox{1.0\textwidth}{!}{
\begin{tabular}{cccccccc}
\hline
                          & \multicolumn{3}{c|}{\begin{tabular}[c]{@{}c@{}}Optimization-based \\white-box attacks\end{tabular}} & \multicolumn{3}{c|}{Black-box attacks}                                              & \multirow{2}{*}{\textbf{Avg.}} \\ \cline{1-7}
Method                    & DF$_{2}$                    & CW$_{2}$                    & \multicolumn{1}{c|}{CW$_{\infty}$}              & Square ($\epsilon$=0.031) & Square ($\epsilon$=0.05) & \multicolumn{1}{c|}{NATTACK} &                                   \\ \hline
Fast-AT \cite{Fast-AT}                   & 0.71                        & 21.3                        & 45.03                                           & 49.0                      & 26.3                     & 68.5                         & 35.14                             \\
YOPO-5-3 \cite{YOPO}                 & 2.13                        & 11.73                       & 33.59                                           & 38.9                      & 17.7                     & 60                           & 27.34                             \\
ATHE \cite{ATHE}                     & 0.42                        & 24.13                       & 48.13                                           & 52.6                      & 33.8                     & 69.5                         & 38.10                             \\
FRL \cite{FRL}                      & 1.82                        & 5.4                         & 21.44                                           & 26.3                      & 9.3                      & 50.5                         & 19.13                             \\
FAT \cite{FAT}                      & 0.48                        & 25.13                       & 48.11                                           & 51.7                      & 32                       & 69                           & 37.74                             \\
ATTA \cite{ATTA}                     & 0.58                        & 21.78                       & 44.97                                           & 46.8                      & 30.9                     & 64                           & 34.84                             \\
TRADES+fFACM-white($\tau$=1) & \textbf{33.66}              & \textbf{27.96}              & \textbf{61.78}                                  & \textbf{76.2}             & \textbf{72.6}            & \textbf{76}                  & \textbf{58.03}                    \\ \hline
\end{tabular}
}
\label{tab3}
\end{table*}

\begin{table*}[t]
\caption{Evaluation (test accuracy(\%)) of different adversarial training methods on CIFAR100 against optimizaiton-based white-box attacks and query-based black-box attacks. fFACM denotes fast-FACM.}
\renewcommand\arraystretch{1.0}
\centering
\footnotesize
\resizebox{1.0\textwidth}{!}{
\begin{tabular}{cccccccc}
\hline
                          & \multicolumn{3}{c|}{\begin{tabular}[c]{@{}c@{}}Optimization-based \\white-box attacks\end{tabular}} & \multicolumn{3}{c|}{Black-box attacks}                                              & \multirow{2}{*}{\textbf{Avg.}} \\ \cline{1-7}
Method                    & DF$_{2}$                    & CW$_{2}$                    & \multicolumn{1}{c|}{CW$_{\infty}$}             & Square ($\epsilon$=0.031) & Square ($\epsilon$=0.05) & \multicolumn{1}{c|}{NATTACK} &                                   \\ \hline
Fast-AT \cite{Fast-AT}                  & 2.49                        & \textbf{17.78}              & 0                                              & 0                         & 0                        & 0                            & 3.38                              \\
YOPO-5-3 \cite{YOPO}                 & 0.82                        & 6.86                        & 20.83                                          & 23.5                      & 11                       & 36                           & 16.5                              \\
ATHE \cite{ATHE}                     & 0.43                        & 10.99                       & 24.82                                          & 26.8                      & 14.8                     & 38.5                         & 19.39                             \\
FRL \cite{FRL}                      & 1.93                        & 2.67                        & 6.94                                           & 8.1                       & 2.8                      & 19                           & 6.91                              \\
FAT \cite{FAT}                      & 0.57                        & 9.03                        & 22.43                                          & 23.6                      & 13.8                     & 36                           & 17.57                             \\
ATTA \cite{ATTA}                     & 0.63                        & 8.17                        & 13.98                                          & 13.9                      & 9.7                      & 20                           & 11.06                             \\
TRADES+fFACM-white($\tau$=3) & \textbf{25.46}              & 12.67                       & \textbf{36.53}                                 & \textbf{50.6}             & \textbf{46.7}            & \textbf{51}                  & \textbf{37.16}                    \\ \hline
\end{tabular}
}
\label{tab4}
\end{table*}

\begin{table*}[t]
\caption{Evaluation (test accuracy(\%)) of different adversarial training methods on MNIST against optimization-based white-box attacks and query-based black-box attacks. fFACM denotes fast-FACM.}
\renewcommand\arraystretch{1.0}
\centering
\footnotesize
\resizebox{1.0\textwidth}{!}{
\begin{tabular}{cccccccc}
\hline
                          & \multicolumn{3}{c|}{\begin{tabular}[c]{@{}c@{}}Optimization-based \\white-box attacks\end{tabular}} & \multicolumn{3}{c|}{Black-box attacks}                                           & \multirow{2}{*}{\textbf{Avg.}} \\ \cline{1-7}
Method                    & DF$_{2}$                    & CW$_{2}$                    & \multicolumn{1}{c|}{CW$_{\infty}$}             & Square ($\epsilon$=0.3) & Square ($\epsilon$=0.4) & \multicolumn{1}{c|}{NATTACK} &                                   \\ \hline
Fast-AT \cite{Fast-AT}                  & \textbf{67.16}              & 51.21                       & 90.61                                          & 70.3                    & 0                       & 95.5                         & 62.43                             \\
YOPO-5-10 \cite{YOPO}                & 33.82                       & 46.36                       & 85.66                                          & 69.6                    & 5.3                     & 90                           & 55.12                             \\
ATHE \cite{ATHE}                     & 1.77                        & 95.23                       & 95.76                                          & 86.2                    & 0                       & 97                           & 62.66                             \\
FRL \cite{FRL}                      & 16.8                        & 73.11                       & 93.8                                           & 67.2                    & 0                       & 94.5                         & 57.57                             \\
FAT \cite{FAT}                      & 0.79                        & 76.06                       & 95.67                                          & 59.3                    & 0                       & 96                           & 54.64                             \\
ATTA \cite{ATTA}                     & 2.5                         & 89.27                       & 96.81                                          & 92.4                    & 0                       & 97                           & 63.00                             \\
TRADES+fFACM-white($\tau$=2) & 40.79                       & \textbf{96.37}              & \textbf{97.03}                                 & \textbf{93.4}           & \textbf{27.9}           & \textbf{97.5}                & \textbf{75.50}                    \\ \hline
\end{tabular}
}
\label{tab5}
\end{table*}

\subsubsection{The Comparison between the fast-FACM model-based TRADES and the Other Adversarial Training Methods}
\label{sec:comparison_fastFACM_and_other_AT}
To verify that TRADES trained classifier with the fast-FACM model is more robust than the other adversarial training methods against optimization-based white-box attacks and query-based black-box attacks, we compare the classification accuracy of the fast-FACM model-based TRADES with the other six adversarial training methods on CIFAR10/100 and MNIST. As shown in Tables~\ref{tab3}, \ref{tab4}, \ref{tab5}, the adversarial training methods generally have low robustness against optimization-based white-box attacks and query-based black-box attacks. In comparison with the other six adversarial training methods, the average improvement of classification accuracy is 19.93\%, 17.77\%, and 12.50\% on CIFAR10/100 and MNIST, respectively. In addition, for Square~\cite{Square}, the increase of the attack strength has little impact on the performance of our method and has a great impact on the other six adversarial training methods. Hence, the results demonstrate that our fast-FACM model can solve the problem of the low robustness of adversarial training methods against optimization-based white-box attacks and query-based black-box attacks.

\begin{table*}[t]
\caption{Evaluation (test accuracy(\%)) of different channel-wise activation suppressing methods on CIFAR10 using various attacks under the white-box setting. CL denotes Clean examples. F, P represent FGSM and PGD attacks, respectively. fFACM denotes fast-FACM. The attacks with gray background color represent the specified attacks.}
\renewcommand\arraystretch{1.0}
\centering
\footnotesize
\resizebox{1.0\textwidth}{!}{
\begin{tabular}{ccccccccccccc}
\hline
                                                                     &                               & \multicolumn{1}{c|}{}   & \multicolumn{7}{c|}{White-box attacks}                                                                                                       & \multicolumn{2}{c|}{\multirow{2}{*}{\begin{tabular}[c]{@{}c@{}}Black-box\\ attacks\end{tabular}}} & \multirow{2}{*}{\textbf{}} \\ \cline{1-10}
                                                                     &                               & \multicolumn{1}{c|}{}   & \multicolumn{4}{c|}{iterative-based}                                & \multicolumn{3}{c|}{optimization-based}                            & \multicolumn{2}{c|}{}                                                                             &                                \\ \cline{1-12}
                                                                     & Method                        & \multicolumn{1}{c|}{Cl} & F              & P              & MI             & \multicolumn{1}{c|}{AA} & \cellcolor{lightgray}DF$_{2}$       & \cellcolor{lightgray}CW$_{2}$       & \cellcolor{lightgray}CW$_{\infty}$ & \cellcolor{lightgray}Sq                                         & \cellcolor{lightgray}NA                              & \cellcolor{lightgray}\textbf{Avg.}                               \\ \hline
\multirow{2}{*}{\begin{tabular}[c]{@{}c@{}}WRN-\\ 16-4\end{tabular}} & TRADES+CAS~\cite{CAS}                   & \textbf{80.92}          & 51.36          & 46.92          & 49.22          & 42.7                    & 0.48           & 20.4           & 44.44                              & 65.2                                       & \textbf{80}                                          & 42.10                          \\
                                                                     & TRADES+fFACM-white($\tau$=1) & 79.42                   & \textbf{55.09} & \textbf{51.63} & \textbf{53.66} & \textbf{47.1}           & \textbf{33.66} & \textbf{27.96} & \textbf{61.78}                     & \textbf{72.6}                              & 76                                                   & \textbf{54.4}                  \\ \hline
\multirow{2}{*}{\begin{tabular}[c]{@{}c@{}}ResNet\\ 18\end{tabular}} & CIFS \cite{CIFS}                         & \textbf{82.46}          & \textbf{61.07} & \textbf{54.66} & \textbf{58.02} & -                       & 0.66           & \textbf{37.99} & 53.74                              & 39.8                                       & 66.5                                                 & 39.74                          \\
                                                                     & +fFACM-white($\tau$=3)   & 82.20                   & 59.48          & 53.98          & 56.51          & -                       & \textbf{9.59}  & 37.43          & \textbf{63.48}                     & \textbf{73.4}                              & \textbf{78}                                          & \textbf{52.38}                 \\ \hline
\end{tabular}
}
\label{tab6}
\end{table*}

\subsubsection{The Comparison between the fast-FACM model and the Channel-wise Activation Suppressing Methods}
\label{sec:comparison_fastFACM_and_CAS_CIFS}
To verify that the fast-FACM model can bring much more robust than the channel-wise activation suppressing methods on the adversarially trained classifier, we compare the classification accuracy of CAS~\cite{CAS} with the fast-FACM model against various attacks on CIFAR10. Besides, to demonstrate that the fast-FACM model can further increase the robustness of the channel-wise activation suppressing methods against optimization-based white-box attacks and query-based black-box attacks, we compare the classification accuracy of CIFS~\cite{CIFS} with and without the fast-FACM model on CIFAR10. As shown in Table~\ref{tab6}, in comparison with CAS, the fast-FACM model has higher robustness against various attacks except for NATTACK~\cite{NATTACK}. In comparison with CIFS, the fast-FACM model can significantly improve the classification accuracy against the specified attacks. The results demonstrate that the fast-FACM can increase the robustness of the channel-wise activation suppressing methods against optimization-based white-box attacks and query-based black-box attacks.


\begin{figure*}[ht]
\centering
\centerline{\includegraphics[width=\textwidth]{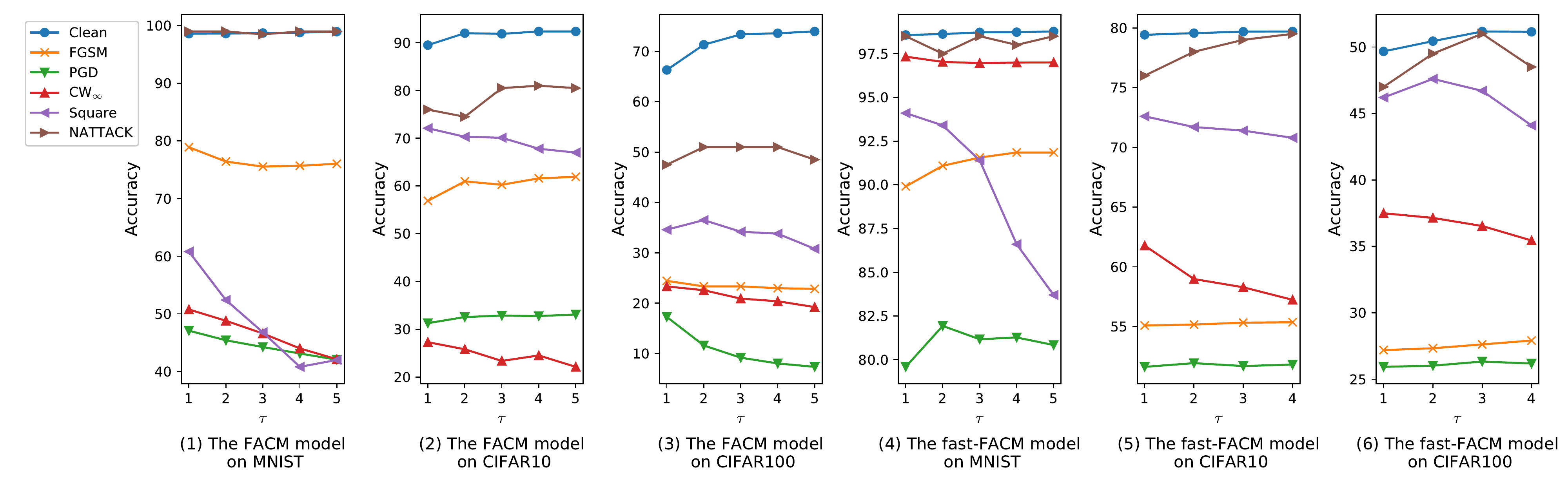}}
\caption{The influence of the size of the parameter $\tau$ in the FACM or fast-FACM model on the test accuracy(\%) against various attacks on MNIST, CIFAR10/100 datasets under the white-box setting. The FACM model is applied to the naturally trained classifier and the fast-FACM model to the TRADES~\cite{TRADES} trained classifier. The attack strength of FGSM and PGD is set as $2/255$ for attacking the naturally trained classifier with the FACM model.}
\label{FIG15}
\end{figure*}

\subsection{The Sensitivity Analysis of $\tau$}
\label{sec:sensitivity_analysis}
Fig.~\ref{FIG15} investigates the influence of the size of the parameter $\tau$ in the FACM and fast-FACM models on the test accuracy of MNIST and CIFAR10/100, respectively. As shown in Fig.~\ref{FIG15}(1)-(3), for the naturally trained classifier with the FACM model, as the parameter $\tau$ becomes large, the test accuracy on clean examples steadily increases, the test accuracy on optimization-based white-box attacks and query-based black-box attacks decreases gradually except for NATTACK increasing first and then decreasing on CIFAR10/100. Besides, the test accuracy on iterative-based white-box attacks increases on CIFAR10 but decreases on MNIST and CIFAR100. Through comprehensive consideration, we set $\tau$ in the FACM model as 3 on CIFAR10/100, and 1 on MNIST.


As shown in Fig.~\ref{FIG15}(4)-(6), for the adversarially trained classifier with the fast-FACM model, as the parameter $\tau$ becomes large, the test accuracy on clean examples and iterative-based white-box attacks steadily increases. Besides, the test accuracy against optimization-based white-box attacks and query-based black-box attacks decreases on CIFAR10 and MNIST. The test accuracy against query-based black-box attacks increases first and then decreases on CIFAR100. Through comprehensive consideration, we set $\tau$ in the fast-FACM model as 1 on CIFAR10, 3 on CIFAR100, and 2 on MNIST.


\begin{table*}[t]
\caption{The attack time (mins) comparison between naturally trained classifier with and without the FACM model on MNIST, CIFAR10/100 under the white-box setting. The longer the attack time, the greater the amount of computation.}
\renewcommand\arraystretch{1.0}
\centering
\footnotesize
\resizebox{1.0\textwidth}{!}{
\begin{tabular}{cccccccccc}
\hline
              & \multicolumn{3}{c}{MNIST}                      & \multicolumn{3}{c}{CIFAR10}                    & \multicolumn{3}{c}{CIFAR100}                   \\ \hline
Method        & FGSM          & PGD40         & CW$_{\infty}$         & FGSM          & PGD20         & CW$_{\infty}$         & FGSM          & PGD20         & CW$_{\infty}$         \\ \hline
Naturally trained classifier & 0.03          & 0.40          & 2.83           & 0.07          & 0.74          & 9.4            & 0.17          & 2.16          & 25.71          \\
+ FACM      & \textbf{0.21} & \textbf{4.89} & \textbf{33.35} & \textbf{0.35} & \textbf{4.89} & \textbf{71.25} & \textbf{0.29} & \textbf{4.67} & \textbf{70.15} \\ \hline
\end{tabular}
}
\label{tab8}
\end{table*}

\begin{table}[t]
\caption{The inference time (mins) comparison between the classifier with and without the FACM or FA model on MNIST, CIFAR10 and CIFAR100.}
\renewcommand\arraystretch{1.0}
\centering
\footnotesize
\begin{tabular}{cccc}
\hline
Method                & CIFAR10 & CIFAR100 & MNIST \\ \hline
TRADES                & 0.051   & 0.050    & 0.028 \\
+fast-FACM($\tau$=2)   & 0.066   & 0.068    & 0.036 \\
+FACM($\tau$=3) & 0.170   & 0.156    & 0.052 \\ \hline
\end{tabular}
\label{tab:inference_time}
\end{table}

\subsection{The Attack Time and Inference Time Comparisons}
\label{sec:build_time_attack_time_inference_time}
The naturally trained classifier with the FACM model can be regarded as a correction model set, which consists of the naturally trained classifier itself and $n-1$ number of FA correction modules and $n$ number of CMPD correction modules. Under the white-box setting, the target of the white-box attacks is an ensemble model, thereby requiring more attack time. As shown in Table~\ref{tab8}, the attack time of the naturally trained classifier with our FACM model is 7-12 times longer on MNIST, 5-8 times on CIFAR10, and 1.7-2.7 times on CIFAR100. Therefore, the computation complexity of the white-box attacks becomes large when these attacks generate adversarial examples against the classifier without the FACM model.

As shown in Table~\ref{tab:inference_time}, in comparison with the original classifier, the inference time of the classifier with the fast-FACM model has a slight increase. The inference time of the classifier with the FACM model is 2-3 times that of the original classifier. Although the inference time of the classifier with the FACM or fast-FACM model increases, it is worth that the performance of both the naturally trained classifier with the FACM model and the adversarially trained classifier with the fast-FACM model is significantly improved. 

\section{Related Work}
\label{sec:related_work}
\textbf{Adversarial Nature:} Szegedy et al.~\cite{linear-property} first proposed the linear property of DNN to explain the fragility of deep learning models against adversarial examples. However, Ding et al.~\cite{Random-forest} found that the negative effect of the added imperceptible noise will be amplified layer by layer till the output of the classifier makes a wrong decision. Inspired by this, we believe the intermediate layer can retain effective features for the real category.

\textbf{Adversarial Training}: Madry et al.~\cite{PGD} proposed the vanilla adversarial training (AT) method. Then, different variants of adversarial training methods were proposed. For example, The Adversarial Training with Hypersphere Embedding (ATHE)~\cite{ATHE} advocated incorporating the hypersphere mechanism into the AT procedure by regularizing the features onto compact manifolds. The Fast Adversarial Training (Fast-AT)~\cite{Fast-AT} trained empirically robust models using a much weaker and cheaper adversary. The Friendly Adversarial Training (FAT)~\cite{FAT} employed confident adversarial data for updating the current model. The Adversarial Training with Transferable Adversarial examples (ATTA)~\cite{ATTA} shows that there is high transferability between models from neighboring epochs in the same training process, which can enhance the robustness of trained models and greatly improve the training efficiency by accumulating adversarial perturbations through epochs. The You Only Propagate Once (YOPO)~\cite{YOPO} reduced the total number of full forward and backward propagation to only one for each group of adversary updates. The Fair-Robust-Learning (FRL)~\cite{FRL} mitigated the unfairness problem that the accuracy of some categories is much lower than the average accuracy of the DNN model. The TRadeoff-inspired Adversarial DEfense via Surrogate-loss minimization (TRADES)~\cite{TRADES} identified a trade-off between robustness and accuracy that serves as a guiding principle in the design of defenses against adversarial examples.

In addition, Sriramanan et al. \cite{GAMA} introduced a relaxation term to the standard loss, which finds more accurate gradient directions to increase attack efficacy and achieve more efficient adversarial training. Wang et al. \cite{OAT} proposed Once-for-all adversarial training methods with a controlling hyper-parameter as the input in which the trained model could be adjusted among different standards and robust accuracies at testing time. Stutz et al. \cite{CCAT} tackled the problem that the robustness generalization on the unseen threat model by biasing the model towards low confidence predictions on adversarial examples. Laidlaw et al. \ cite {Perceptual-adversarial-training} developed perceptual adversarial training against all imperceptible attacks. Pang et al. \cite{Bag-of-tricks-for-AT} provided comprehensive evaluations on CIFAR10, which investigate the effects of mostly overlooked training tricks and hyperparameters for adversarially trained DNN models. Cui et al.~\cite{Learnable-boundary} used a clean model to guide the adversarial training of the robust model, which inherits the boundary of the clean model. Misclassification-aware adversarial training~\cite{MART} differentiated the misclassified examples and correctly classified examples.

Recently, Bai et al. \cite{CAS} and Yan et al. \cite{CIFS} investigated the adversarial robustness of DNNs from the perspective of channel-wise activations. They found adversarial training~\cite{PGD} can align the activation magnitudes of adversarial examples with those of their natural counterparts, but the over-activation of adversarial examples still exists. To further improve the robustness of DNNs, Bai et al.~\cite{CAS} proposed Channel-wise Activation Suppressing (CAS) to suppress redundant activation of adversarial examples, and Yan et al.~\cite{CIFS} proposed Channel-wise Importance-based Feature Selection (CIFS) to suppress the channels that are negatively relevant to predictions. However, these methods also remain to achieve poor performance against both optimization-based white-box attacks and query-based black-box attacks.

\textbf{Adversarial Training with External Data.} External data was verified to reduce the gap between the robust and natural accuracies~\cite{robust-vision-transformers,defending-against-image-corruptions,in-out-distribution,exploring-architectural-ingredients,reducing-excessive-margin,stable-neural-ODE,revisiting-residual-networks,uncovering-the-limits,fixing-data-augmentation}. However, data augmentation, including image-to-image learning generators~\cite{many-faces-of-robustness,generated-data,proxy-distributions,better-diffusion-models,AdversarialAugment} and other input transformation methods~\cite{PRIME,mixup,common-data-augmentation,AugMix,RandAugment,NoisyMix,DAJAT}, was an effective technique to generate more data to avoid the robust overfitting without external data. Additionally, network architecture~\cite{RobustResBlock,stable-neural-ODE,exploring-architectural-ingredients,a-winning-hand} was explored to find the relationship with the robustness against adversarial examples. Finally, three robust benchmark~\cite{RobustBench,RobustArt,benchmarking-and-rethinking} were conducted to compare the performance of each defense model.

\textbf{Randomization:} These~\cite{SAP,Mitigating-adversarial-effects,RS,SLQ,Adversarial-noise-layer,Random-self-ensemble,Randomized-diversification,SORS} are a class of methods to defend against adversarial examples and keep the accuracy of the classifier on clean samples. For example, random image resizing and padding (RP)~\cite{Mitigating-adversarial-effects} and stochastic local quantization (SLQ)~\cite{SLQ} defend against adversarial examples by adding a randomization layer before the input to the classifier. However, Athalye~\cite{Obfuscated-gradients} claimed that Although RP~\cite{Mitigating-adversarial-effects} and SLQ~\cite{SLQ} can be broken through expectation over transformation, these methods still have the ability to defend against optimization-based white-box attacks~\cite{CW,DeepFool} and query-based black-box attacks~\cite{Square,NATTACK} relative to iterative-based gradient attacks~\cite{PGD}.

\section{Conclusion}
\label{sec:conclusion}
In this paper, the Feature Analysis and Conditional Matching prediction distribution (FACM) model is proposed to improve the robustness of the DNN-implemented classifier against optimization-based white-box attacks and query-based black-box attacks. Specifically, the \textit{Correction Property} is proposed and proved that the features retained in the intermediate layers of the DNN can be used to correct the classification of the classifier on adversarial examples. Based on the \textit{Correction Property}, the FA and CMPD modules are proposed to be collaboratively implemented in our model to enhance the robustness of the DNN classifier. Then, the \textit{Diversity Property} is proved to exist in the correction modules. A decision module is proposed in our model to further enhance the robustness of the DNN classifier by utilizing the diversity among the correction modules. The experimental results show that our FACM model can effectively improve the robustness of the naturally trained classifier against various attacks, especially optimization-based white-box attacks, and query-based black-box attacks. Our fast-FACM can improve the test accuracy of the adversarially trained classifier against optimization-based white-box attacks and query-based black-box attacks on basis of keeping the performance on clean examples and other attacks. In addition, our model can be well combined with randomization methods. And the classifier with our model is equal to an ensemble model with the size of $2n$, thereby increasing the computational complexity of white-box attacks under the white-box setting.


\bibliographystyle{cas-model2-names}

\bibliography{ref}

\end{document}